\documentclass[11pt]{amsart}

\usepackage{Packages}
\usepackage{New_commands}
\usepackage{Style}
\usepackage{comment}
\usepackage{times}
\usepackage{usual_macros}
\usepackage{dsfont}
\usepackage{comment}
\usepackage{enumitem}
\usepackage{tikz}
\usepackage{stmaryrd}
\usepackage{algpseudocode}
\usepackage{algorithm2e} 
\usepackage[foot]{amsaddr}

\title{Kernelized Diffusion Maps}

\author{Loucas Pillaud-Vivien$^{1, 2}$}
\address{$^1$Courant Institute of Mathematical Sciences, New York University, New York.}
\address{$^2$Center for Computational Mathematics, Flatiron Institute, Simons Foundation, New York.}

\author{Francis Bach$^3$}
\address{$^3$Inria, Ecole Normale Supérieure,
PSL Research University.}

\definecolor{Bea_1}{RGB}{71,127,124}
\definecolor{Bea_2}{RGB}{128, 181, 184}
\definecolor{Bea_3}{RGB}{15, 158, 179}
\definecolor{Bea_4}{RGB}{210, 77, 4}

\begin{document}

\begin{abstract}%
   Spectral clustering~\cite{ng2001spectral} and diffusion maps~\cite{COIFMAN20065} are celebrated dimensionality reduction algorithms built on eigen-elements related to the diffusive structure of the data. The core of these procedures is the approximation of a Laplacian through a graph kernel approach~\cite{hein2007graph}, however this local average construction is known to be cursed by the high-dimension $d$. In this article, we build a different estimator of the Laplacian, via a reproducing kernel Hilbert space method, which adapts naturally to the regularity of the problem. We provide non-asymptotic statistical rates proving that the kernel estimator we build can circumvent the curse of dimensionality. Finally we discuss techniques (Nyström subsampling, Fourier features) that enable to reduce the computational cost of the estimator while not degrading its overall performance.
\end{abstract}

\maketitle

\section{Introduction}

One of the reasons of the success of learning with reproducing kernel Hilbert spaces (RKHS) is that they naturally select problem-adapted bases of test functions. Even more interestingly, leveraging the underlying regularity of the target function, RKHSs have the ability to circumvent the curse of dimensionality. This is exactly where all techniques resting on local averages fail: approximating a problem will always be cursed by the high-dimension $d$, because one will need $n^{-1/d}$ points to perform well. This main difference echoes in the nature of the kernels: \emph{pointwise positive kernels} in the non-parametric estimation literature~\cite{nadaraya1964estimating} and \emph{positive semi-definite (PSD) kernels} in modern kernel learning~\cite{steinwart2008support,smola-book}.

Solving a problem with PSD kernels that used to be tackled with local techniques is at the heart of this work. Indeed, we estimate the diffusion operator (or Laplacian) related to a measure $\mu$ through its principal eigen-elements. When cast into an \emph{unsupervised} learning problem, this can be seen as a dimensionality reduction technique resting on the diffusive nature of the data. This is the core of the celebrated spectral clustering algorithm~\cite{von2007tutorial} and of diffusion maps~\cite{COIFMAN20065} in the context of molecular dynamics. However, as introduced before, these algorithms are based on graph Laplacians; an intrinsically local construction that scales poorly with the dimension~\cite{hein2007graph} and does not benefit from all the recent works on PSD kernels that tackle potential high-dimensional settings~\cite{martinsson2020randomized, meanti2020kernel}.

Let us explain the fundamental difference between the approach of this work and that of graph Laplacians. When we want to estimate the diffusion operator (or its eigenvectors)
\begin{align}
\mathcal{L}:= -\Delta + \langle \nabla V, \nabla \cdot \rangle,
\end{align}
one of the difficult aspects is to approximate differential operators. While currently, people use local kernel smoothing techniques, our approach is different. It leverages the reproducing property of derivatives in RKHS and the self-adjointness of $\mathcal{L}$ to circumvent this difficulty: this strategy has shown fruitful results in numerical analysis for partial differential equations, where it is called \emph{meshless methods}~\cite{schaback2006kernel}.

In another direction, it is interesting to note that~\cite{salinelli1998} tried to show that considering the first eigenvectors of $\mathcal{L}$ was \emph{the good way} of generalizing the principal components analysis procedure~\cite{pearson1901liii,hotelling1933analysis} \textit{in a non-linear fashion}. At this time, (i) neither the theory behind diffusions and weighted Sobolev spaces (ii) nor the theory of RKHS were mature. Hence, he clearly explained (i) that the theoretical framework of his analysis was limited but could be extended, and (ii) that at this point solving numerically the problem was impossible in high-dimension as it necessitates to discretize the Laplacian. Quite surprisingly, the literature on graph Laplacians seems to have overlooked Salinelli's seminal contribution. Our work can be considered as a natural continuation of his: pushing further the theoretical comprehension of this \emph{non-linear principal component analysis} with modern tools and giving a way to solve it efficiently.

Note also that this work has been motivated by applications in \textit{molecular dynamics} where diffusion maps is an important dimensional reduction technique to find reaction coordinates, i.e., the slow diffusion modes of the high-dimensional dynamics~\cite{Boaz2006}. This article can also be seen as a natural extension of~\cite{pillaud2020statistical} whose aim was to estimate the first non-zero eigenvalue of $\mathcal{L}$ (this is saying, its spectral gap). Besides being more mature, the focus of this work is quite different: we focus here on the estimation of the whole spectrum of $\mathcal{L}$ and try to be more precise regarding its convergence properties. Finally remark that the procedure we are going to describe can be seen as a data estimation of the \textit{Koopman generator} of the dynamics generated by $\mathcal{L}$~\cite{klus2020data}, that leverages crucially it self-adjointness property.  

\section{Diffusion operator}
\label{sec:DO}

Consider a probability measure $d\mu$ on $\R^d$ which has a density with respect to the Lebesgue measure and can be written under the following form: $d\mu(x) = e^{-V(x)}dx$, where $V$ is called the \textit{potential function}. Consider $H^1(\mu)$ the subspace of functions of $L^2(\mu)$ (i.e., which are square integrable) that also have all their first order derivatives in $L^2$, that is, $H^1(\mu) = \{f \in L^2(\mu),\ \int_{\R^d}f^2 d\mu+\int_{\R^d}\|\nabla f\|^2 d\mu < \infty \}$, where $\nabla f$ is the gradient of $f$ and $\|\cdot\|$ the standard Euclidean norm. 

The aim of this work is to estimate the diffusion operator $\mathcal{L}$, associated with measure $\mu$, given access to $x_1,\hdots,x_n$, i.i.d. samples distributed according to $\mu$. It is defined by 
\begin{align}
\label{eq:opeartor_L}
\mathcal{L} \phi := -\Delta \phi + \nabla V \cdot \nabla \phi,
\end{align}
where $\phi$ is a smooth enough test function.
\subsection{Langevin diffusion}

Let us consider the overdamped Langevin diffusion in $\R^d$, that is the solution of the following stochastic differential equation:
\begin{align}
\label{eq:langevin_dimensionality_reduction}
\mathrm{d}X_t = -\nabla V (X_t) \mathrm{d}t + \sqrt{2}\,\mathrm{d} B_t,
\end{align}
where $(B_t)_{t\geqslant0}$ is a $d$-dimensional Brownian motion. It is well-known \cite{bakry2014} that the law of $(X_t)_{t\geqslant0}$ converges to the Gibbs measure $d\mu$ and that the Poincaré constant (see Remark~\ref{prop:link_poinca} below) controls the rate of convergence to equilibrium in $L^2(\mu)$. Let us denote by $P_t (f) $ the Markovian semi-group associated with the Langevin diffusion $(X_t)_{t\geqslant0}$. It is defined in the following way: $P_t (f) (x) = \mathbb{E}[f(X_t)| X_0 = x]$. This semi-group satisfies the dynamics 
$$ \frac{d } {dt} P_t (f)= -\mathcal{L} P_t (f), $$
where $\mathcal{L} \phi = -\Delta \phi + \nabla V \cdot \nabla \phi$ is a differential operator called the infinitesimal generator of the Langevin diffusion \eqref{eq:langevin_dimensionality_reduction} ($\Delta$ denotes the standard Laplacian on $\R^d$). Note that by integration by parts, the semi-group $(P_t)_{t \geqslant 0}$ is reversible with respect to $d\mu$, that is: $\int f(\mathcal{L}g)\,d\mu = \int \nabla f \cdot \nabla  g\, d\mu = \int (\mathcal{L}f)g\,d\mu$. This also shows that $\mathcal{L}$ is a symmetric positive definite operator on $H^1(\mu)$.
\begin{remark}[Link with Poincaré constant]
\label{prop:link_poinca}
Let us call $\pi$ the orthogonal projector of $L^2(\mu)$ on constant functions: $ \pi f : x \in \R^d \mapsto \int f d\mu$ and define $L_0^2(\mu) := \ker\, \pi$.  Under Assumption~\ref{ass:spectral_gap} (see below) the first non-zero eigenvalue of $\mathcal{L}$ is:
\begin{align}
\label{eq:poinca_operator_dimensionality_reduction}
\Poinca^{-1} = \inf_{f \in (H^1(\mu) \cap L_0^2(\mu)) \setminus \{0\}}\frac{\langle f,\mathcal{L} f \rangle_{L^2(\mu)}}{\|f\|_{L^{2}(\mu)}^2},
\end{align}
where $\mathcal{P}$ is also known as the Poincaré constant of the distribution $d\mu$~\cite{Ane2000}.
\end{remark}
\subsection{Some useful properties of the diffusion operator}
\label{subsec:properties}
\subsubsection*{Positive semi-definiteness.} The first property that we saw is symmetry and positiveness of~$\mathcal{L}$ in~$H^1(\mu)$. It comes from the following integration by part identity:
\begin{align}
\label{eq:symmetry}
\int f(\mathcal{L}g)\,d\mu = \int \nabla f \cdot \nabla  g\, d\mu = \int (\mathcal{L}f)g\,d\mu,
\end{align}
showing that the quadratic form induced by $\mathcal{L}$ is also the Dirichlet energy $$\left\langle \mathcal{L} f, f \right\rangle_{L^2(\mu)}= \int \|\nabla f\|^2 d\mu=:\mathcal{E} (f). $$
\subsubsection*{Link with Schrödinger operator.} In the field of partial differential equations (PDEs) we say that an operator is of Schrödinger type if it is the sum of the Laplacian and a multiplicative operator, this comes from the fact that this is the type of operator that governs the dynamics of quantum systems~\cite{Helffer2005witten}. Here, let us define the Schrödinger operator $\widetilde{\mathcal{L}} := -\Delta + \mathcal{V}$, where $ \mathcal{V}:= \frac{1}{2} \Delta V - \frac{1}{4} \|\nabla V\|^2$. We can show that $\widetilde{\mathcal{L}}$ and $\mathcal{L}$ are conjugate to each other: indeed, a rapid calculation shows that 
\begin{align*}
\widetilde{\mathcal{L}} = e^{-V/2}\mathcal{L}\left[e^{V/2}\ \cdot\right].
\end{align*} 
As Schrödinger operators are well-studied, we can infer from this fact interesting properties on the spectrum of $\mathcal
{L}$. Indeed, 
\begin{align*}
(\lambda, u) \textrm{\ eigen-elements\ of\ } \widetilde{\mathcal{L}}\ \  \Leftrightarrow\ \   (\lambda, e^{V/2}u) \textrm{\ eigen-elements\ of\  }  \mathcal{L},
\end{align*}
and we also have the following equality for $f$ smooth enough:
\begin{align*}
 \frac{\langle \mathcal{L} f, f \rangle_{L^2(\mu)}}{\|f\|^2_{L^2(\mu)}} = \frac{\langle \widetilde{\mathcal{L}} f, f \rangle_{L^2(\R^d)}}{\|f\|^2_{L^2(\R^d)}}. 
\end{align*}
\subsubsection*{Spectrum of $\mathcal{L}$.} The most important property that we can infer from this is the nature of the spectrum of $\mathcal{L}$. Indeed, it is well known~\cite{reed2012methods} that if $\mathcal{V}$ is locally integrable, bounded from below and coercive ($\mathcal{V}(x) \longrightarrow +\infty$, when $\|x\|\to +\infty$), then the Schrödinger operator has a compact resolvent. In particular, we will assume throughout the article the following
\begin{assumption}[Spectrum of $\mathcal{L}$]
\label{ass:spectral_gap}
Assume that $\frac{1}{2} \Delta V(x) - \frac{1}{4} \|\nabla V\|^2 \longrightarrow +\infty$, when $\|x\|\to +\infty$.
\end{assumption}
\noindent Assumption~\ref{ass:spectral_gap} implies that $\mathcal{L}$ has a compact resolvent. This also implies that $\mathcal{L}$ has a purely discrete spectrum and a complete set of eigenfunctions. Note that this assumption implies also a spectral gap for the diffusion operator~$\mathcal{L}$ and hence that a Poincaré inequality holds. Throughout this work and even if not clearly stated, we will assume Assumption~\ref{ass:spectral_gap}. For further discussions on the spectrum of~$\mathcal{L}$, we refer to~\cite{bakry2014,Helffer2005witten}.

\section{Approximation of the diffusion operator in the RKHS}
\label{sec:Estimation_dimensionality_reduction}
Let $(\h, \langle \cdot, \cdot\rangle_{\h})$ be an RKHS with positive definite kernel $K$. Let us suppose the following:
\begin{assumption}[Universality] 
\label{ass:universality}
$\h$ is dense in $H^1(\mu)$.
\end{assumption}
Note that this is the case for most of the usual couples kernels/distribution: Gaussian, exponential kernels are universal \cite{micchelli2006universal} if $\mu$ has compact support or subgaussian tails. As the expression of the diffusion operator in Eq.~\eqref{eq:opeartor_L} involves derivatives of test functions, we will also need some regularity properties of the RKHS. Indeed, to represent $\nabla f$ in our RKHS we leverage crucially the partial derivative reproducing property of the kernel space. For this, we need:
\begin{assumption}[Smoothness]
\label{ass:smoothness}
$K$ is a positive definite kernel such that $K \in  \mathcal{C}^2(\R^d \times \R^d)$.
\end{assumption}
For $i \in \llbracket1, d\rrbracket$, denote by $\partial_i = \partial_{x^i}$ the partial derivative operator with respect to the $i$-th component of $x$. It has been shown \cite{Zhou2008} that under Assumption~\ref{ass:smoothness}, we can define $\h \ni \partial_i K_x: y \to\partial_{x^i} K (x,y) $ and that a partial derivative reproducing property holds true: $\forall f \in \h$ and $\forall x \in \R^d$, $\partial_i f(x) = \langle \partial_i K_x  , f \rangle_\h $. Hence, thanks to Assumption~\ref{ass:smoothness}, $\nabla f$ is easily represented in the RKHS. We also need some boundedness properties of the kernel.
\begin{assumption}[boundedness]
\label{ass:boundedness}
$K$ is a kernel such that $\forall x \in~\R^d, \,K (x,x) \leqslant \mathcal{K}$ and\footnote{The subscript $d$ in $\mathcal{K}_d$ accounts for the fact that this quantity is expected to scale linearly with $d$ (Gaussian kernel).} $ \left\|\nabla K_x\right\|^2 \leqslant \mathcal{K}_d$, where $\left\|\nabla K_x \right\|^2: = \sum_{i=1}^d\langle \partial_i K_x, \partial_i K_x \rangle = \sum_{i=1}^d \frac{\partial^2 K}{\partial x^i \partial y^i} (x,x)$ (see calculations below), $x$ and $y$ standing respectively for the first and the second variables of $(x,y) \mapsto K(x,y)$.
\end{assumption}
The equality in the expression of $\|\nabla K_x\|^2$ arises from the following computation: for all $x, y \in \R^d$, $\langle \partial_i K_x, \partial_i K_y \rangle = \partial_{x^i} \left( \partial_i K_y (x) \right) = \partial_{x^i} \partial_{y^i}  K(x,y)  $. Note that, for example, the Gaussian and exponential kernels satisfy Assumptions \ref{ass:universality}, \ref{ass:smoothness}, \ref{ass:boundedness}. Boundedness is stated here for the sake of clarity, however, up to logarithmic terms, the results of this paper would hold if we let $\|K_X\|^2, \|\nabla K_X\|^2$ be subgaussian random variables. 
\begin{example}[Gaussian kernel]
  \label{ex:rbf} 
  A prototypical example is the Gaussian kernel (or radial basis
  function), with bandwidth $\sigma > 0$, for which we can compute, for $i\neq j$,
  \begin{align*}
    K(x, y) &= \exp\left(-\frac{\|x-y\|^2}{2\sigma^2}\right),\hspace{0.7cm}
    \partial_{x^i} \partial_{y^j}  K(x,y) = - \frac{(x_i - y_i)(x_j -
      y_j)}{\sigma^4} K(x, y),
    \\
    \partial_{x^i} K(x, y) &= -\frac{(x_i - y_i)}{\sigma^2} K(x, y),
    \qquad
    \partial_{x^i} \partial_{y^i} K(x, y) = \left(\frac{1}{\sigma^2} -
      \frac{(x_i - y_i)^2}{\sigma^4}\right) K(x, y).
  \end{align*}
\end{example}
\subsection{Embedding the diffusion operator in the RKHS}

Let us define the following operators from $\h$ to $\h$: 
\begin{align}
\Cov &= \E_\mu \left[K_X \otimes K_X \right],\hspace*{1.5cm} \D = \E_\mu \left[\nabla K_X \otimes_d \nabla K_X \right],
\end{align}
where $\otimes$ is the standard tensor product: $\forall f,g, h \in \h$, $(f \otimes g) (h) = \langle g, h \rangle_{_\h} f$ and $\otimes_d$ is defined as follows:  $\forall f,g \in \h^d$ and $h \in \h$, $(f \otimes_d g) (h) = \sum_{i=1}^d \langle g_i, h\rangle_{_\h} f_i $. By the reproducing property of $K$, $(\h, \langle\cdot,\cdot\rangle_\h)$ injects canonically in $(L^2(\mu), \langle\cdot,\cdot\rangle_{L^2(\mu)})$ through an operator $\mathsf{S}$, together with its adjoint $\mathsf{S}^*$ defined from $L^2(\mu)$ to $\h$ such that for all $x \in \R^d$:
\begin{align*}
\forall f \in \h, \ \  \mathsf{S}f (x) &= \langle f, K_x\rangle_\h = f(x) ,\hspace*{1cm}  \forall f \in L^2(\mu),\ \  \mathsf{S}^*f (x) = \E_\mu \left[K(x,X)f(X) \right].
\end{align*}
Note that $\mathsf{S}^*\mathsf{S} = \Cov$. 
With these definitions, and thanks to the symmetry property of $\mathcal{L}$ derived in  Eq.~\eqref{eq:symmetry}, we can represent the diffusion operator $\mathcal{L}$ in the RKHS.
\begin{proposition}[Embedding of $\mathcal{L}$ in the RKHS]
\label{prop:embedding_L}
Suppose Assumptions \ref{ass:universality},\ref{ass:smoothness} hold, then
\begin{align}
\label{eq:diff_operator_in_rkhs}
\D = S^* \mathcal{L} S,
\end{align}
where the equality stands for the equality between operators of $\h$.
\end{proposition}
\begin{proof}
For $z \in \R^d$, $f \in \h$,
\begin{align*}
\langle \D f, K_z \rangle_\h &= \int \nabla f(x) \cdot \nabla_x K(x,z) d\mu(x) \\
&= -\int \Delta f(x) K(x,z) d\mu(x) + \int \nabla f(x) \cdot \nabla V(x)  K(x,z) d\mu(x)  \\
&= \langle \mathcal{L} \mathsf{S} f, \mathsf{S} K_z \rangle_{L^2(\mu)} \\
&= \langle \mathsf{S}^*\mathcal{L} \mathsf{S} f, K_z \rangle_\h,
\end{align*}
hence the equality between operators. Note that to go from the first line to the second ones, we used the symmetry of $\mathcal{L}$.
\end{proof}
We want to construct an approximation of the eigen-elements of $\mathcal{L}$ with domain $H^1_0(\mu):=H^1(\mu) \cap L_0^2(\mu)$, where we recall that $L_0^2(\mu) = \ker \, \pi$ stands for the space of square integrable functions without the constants. Similarly, let us denote $\h_0 = (\ker\, \D)^\perp$, the subspace of $\h$ without the constant functions. Note that this operator is invertible as a consequence of the spectral gap Assumption~\ref{ass:spectral_gap}. In the following we will approximate the eigen-elements of $\mathcal{L}^{-1}$. First we give a representation of $\mathcal{L}^{-1}$ in the RKHS $\h$, then we construct an operator on $\h$ that has the same eigen-elements of $\mathcal{L}^{-1}$. Indeed, if we denote $\D^{-1}$ the inverse of $\D$  restricted on $\left(\ker \D\right)^\perp$, we have:
\begin{proposition}[Representation of $\mathcal{L}^{-1}$]
\label{prop:representation_diffusion}
Suppose Assumptions \ref{ass:spectral_gap},\ref{ass:universality},\ref{ass:smoothness} hold, then
\begin{align}
\mathcal{L}^{-1} = \mathsf{S} \D^{-1} \mathsf{S}^*,
\end{align}
where the equality stands for the equality between operators whose domains are $H^1_0(\mu)$.
\end{proposition}
Thanks to Proposition~\ref{prop:representation_diffusion}, we have a representation of $\mathcal{L}^{-1}$ in the RKHS \textit{through} the embedding~$S$. But what we really would like is an operator on $\h$ that as the same eigen-elements as $\mathcal{L}^{-1}$. Such a representation allows for numerical computations: this is the purpose of the following proposition.
\begin{theorem}[Eigen-elements of $\mathcal{L}^{-1}$ as functions in the RKHS]
\label{prop:eigen-elements_diffusion}
Decompose the inverse of the diffusion operator such that $\mathcal{L}^{-1} = \mathsf{S} \D^{-1} \mathsf{S}^* = \mathsf{S} \D^{-1/2} \D^{-1/2} \mathsf{S}^*$, then, 
\begin{enumerate}[label=(\roman*), itemsep = 0pt, topsep=5pt]
\item $\mathsf{S} \D^{-1/2}$ is a bounded operator from $\h_0$ to $H^1_0(\mu)$.
\item $\ \D^{-1/2} \CCov \D^{-1/2}$ is a self-adjoint compact operator on $\h_0$  with the same spectrum as $\mathcal{L}^{-1}$.
\item If $\lambda \neq 0$ is an eigenvalue of $\ \D^{-1/2} \CCov \D^{-1/2}$  with eigenvector $u \in \h_0$, then $\lambda$ is an eigenvalue of $\mathcal{L}^{-1}$ with eigenvector $\ \mathsf{S} \D^{-1/2}  u \in H^1_0(\mu)$.
\end{enumerate}
\end{theorem}
This theorem will allow us to approximate the eigen-elements of $\mathcal{L}^{-1}$ with the ones of the operator $\ \D^{-1/2} \CCov \D^{-1/2}$ (that is well-defined only on $\h_0$) with a finite set of samples. Its proof is the consequence of the representation of $\mathcal{L}^{-1}$ presented in the previous proposition and a technical lemma on Hilbert operators proven in Appendix (Lemma~\ref{le:reversement_operator}).

\subsection{Definition of the estimator}

\subsubsection*{Empirical operators.} We define the empirical counterpart of $\D$ and $\Cov$: they are defined by replacing expectation with respect to $\mu$ by expectations with respect to its empirical measure~$\widehat{\mu}_n = \frac{1}{n} \sum_{i = 1}^n \delta_{x_i}$ where $x_1, \hdots, x_n$ are i.i.d. samples distributed according to~$d\mu$. 
\begin{align}
 \Covn  = \frac{1}{n} \sum_{i=1}^n K_{x_i} \otimes K_{x_i}, \hspace*{0.15cm}\text{ and } \hspace*{0.35cm} \Dn  = \frac{1}{n} \sum_{i=1}^n \nabla K_{x_i} \otimes_d \nabla K_{x_i}.
\end{align}
Hence, one could be tempted to define our estimator as $\Dn^{-1/2}\CCovn\Dn^{-1/2}$. However, this definition carries two main problems:
\begin{enumerate}[label=(\roman*), itemsep = 0pt, topsep=5pt]
\item If $f \in \ker\ \Dn$, i.e., for all $i\leqslant n$, $\nabla f(X_i) = 0$, then $\|\Dn^{-1/2}\CCovn\Dn^{-1/2} f\| = + \infty$. This is an \emph{overfitting-type issue}.
\item Another problem is related to the fact that finding the eigen-elements of $ \D^{-1/2} \CCov \D^{-1/2}$ is equivalent to solving the generalized eigenvalue problem: $\ \CCov f = \sigma \D f $. Such systems are known to be numerically unstable as mentioned by \cite{crawford1976stable}. This would be especially the case when replacing the operators by their empirical counterpart. This is a \emph{stability issue}.
\end{enumerate}
\subsubsection*{Regularization.} These two concerns recall the pitfall of overfitting for regression tasks. Hence, as for kernel ridge regression, a natural idea is to regularize with some parameter $\lambda$. This leads to the following definition of our estimator and its empirical counterpart:
\begin{definition}[Definition of the estimator]
Under Assumptions~\ref{ass:spectral_gap},\ref{ass:universality},\ref{ass:smoothness},\ref{ass:boundedness}, we define the two estimators of the inverse diffusion operator $\mathcal{L}^{-1}$: 
\begin{align}
\hspace*{1cm}\hspace*{-5cm}\textrm{Biased\ estimator: } \hspace*{1cm} (\D + \lambda I)^{-1/2}\CCov(\D + \lambda I)^{-1/2} \\
\hspace*{1cm}\hspace*{-5cm}\textrm{Empirical\ estimator: } \hspace*{1cm} (\Dn + \lambda I)^{-1/2}\CCovn(\Dn + \lambda I)^{-1/2} \hspace*{-0.05cm}.
\end{align}
\end{definition}
In the following, to shorten notations, let us define $\D_\lambda = \D + \lambda I$ and $\Dn_\lambda = \Dn + \lambda I$. Obviously, the main drawback of this regularization is that it induces a bias in our estimation: more precisely the acute reader will recognize that the bigger the $\lambda$ the closer the problem is to kernel-PCA~\cite{mika1999kernel}. In other words, the scale of $\lambda$ controls the magnitude of the diffusive information we want to retrieve from the data (this point of view can be further studied but we leave this for future work at this point).

When analyzing the performances of our empirical estimator, we will draw a particular attention to the comparison with the standard algorithm that computes the eigen-elements of the operator: diffusion maps~\cite{COIFMAN20065,hein2007graph}. We emphasize that the RKHS method we present allows to benefit from the numerous positive aspects of RKHS methods~\cite{smola-book}: both \textit{on the statistical side} regarding the dependency on the dimension, the adaptivity to the regularity of the target \cite{caponnetto2007optimal}, and on \textit{on the computational side} benefiting from the techniques developed in the literature like column subsampling or the use of random features \cite{martinsson2020randomized}.
\subsection{What quantities are we interested in approximating?}
\subsubsection*{Requirements of the problem.} The natural and general goal of the present work is to give an approximation of the diffusion operator based on i.i.d. samples. However, there are in fact more precise practical objects that the reader may want to have an approximation of:
\begin{itemize}[itemsep=0pt, topsep=5pt]
\item {\bfseries The whole operator.} Either its representation in $\h$ either in $H^1(\mu)$. This can lead, as recalled in Subsection~\ref{subsec:properties}, to the estimation of Schrödinger operators. This can also be used to regularize a semi-supervised problem with the Dirichlet energy of the unlabeled data to leverage its structure~\cite{cabannes2021overcoming,cabannes2022minimal}.
\item {\bfseries The semigroup.} In fact, as $\mathcal{L}$ is the infinitesimal generator of the dynamics, we can be interested in the convergence to the associated semigroups $e^{t\mathcal{L}}$~\cite{klus2020data}. 
\item {\bfseries Eigenvectors.} As one of the main applications of this estimator could be the computation of a low-dimensional embedding of the data through the eigenvectors of $\mathcal{L}$, we are directly interested in the approximation of the eigenvectors. Either eigenvector per eigenvector, either finite dimensional subspaces spanned by few of them. Note that we are mostly interested in the small eigenvalues of $\mathcal{L}$, corresponding to the large eigenvalues of $\mathcal{L}^{-1}$, because they are those governing the behaviour of the dynamics~\cite{Lelievre2013}. 
\item {\bfseries Eigenvalues.} As it has already been done in previous work for the top eigenvalue~\cite{pillaud2020statistical}, one would like to approximate a set of eigenvalues. Another application is the construction of the diffusion distance used for clustering~\cite{COIFMAN20065}.
\end{itemize}
\subsubsection*{Previous results.} In previous works, e.g., \cite{hein2007graph} and~\cite{COIFMAN20065} proved the convergence of the estimated operator. However, note that the convergence theorems are given \emph{pointwise}, for \emph{bounded domains} and have a \emph{bad dependency in the dimension} as $n^{-1/d}$. We will try to overpass these three limiting results. Please note that the operator norm convergence to the diffusion operator \emph{implies the convergence of all the quantities mentioned earlier:} {\bfseries (i) semigroup} at finite time, thanks to the inequality: $\|e^{B}-e^{A}\| \leq \|B-A\|e^{\max\{\|A\|,\|B\|\}}$, {\bfseries (ii) eigenvectors and eigenvalues,} directly by perturbation theory arguments. Importantly, refined bounds are discussed if one want to approximate $k$-dimensional subspaces, similarly to \cite{zwald2005convergence}.
\section{Statistical analysis of the estimator}
As said earlier, to shorten the notations, let us define for an operator $A$, the operator $A_\lambda := A + \lambda I$. We will split the problem in two: a bias term and a variance term
\begin{align*}
\left\|\widehat{\D}_\lambda^{-1/2} \CCovn \widehat{\D}_\lambda^{-1/2} - \D^{-1/2} \CCov \D^{-1/2}\right\| \leqslant \underbrace{\left\|\widehat{\D}_\lambda^{-1/2} \CCovn \widehat{\D}_\lambda^{-1/2} - \D_\lambda^{-1/2} \CCov \D_\lambda^{-1/2}\right\|}_{\mathrm{variance}}  + \underbrace{\left\| \D_\lambda^{-1/2} \CCov \D_\lambda^{-1/2} - \D^{-1/2} \CCov \D^{-1/2}\right\|}_{\mathrm{bias}}
\end{align*}
The variance term corresponds to the statistical error coming from the fact that we have only access to a finite set of $n$ samples of the distribution $\mu$. The bias comes from the introduction of a regularization of the operator $\D$ scaled by $\lambda$. We first derive bounds for the variance term.
\subsection{Variance analysis}
\begin{proposition}[Analysis of the statistical error]
\label{prop:hat_P_to_P_lambda_dimensionality_reduction}
Suppose Assumptions~\ref{ass:spectral_gap},\ref{ass:universality},\ref{ass:smoothness},\ref{ass:boundedness}, hold true. For any $\delta \in (0,1/3)$, $0<\lambda \leqslant\|\D\| $ and any integer $n \geqslant 15 \frac{\mathcal{K}_d}{\lambda} \log \frac{4\,\mathrm{Tr } \D }{\lambda \delta }$, with probability at least $1-2\delta$,
\begin{align}
\label{eq:hat_P_to_P_lambda_dimensionality_reduction}
\left\|\widehat{\D}_\lambda^{-1/2} \CCovn \widehat{\D}_\lambda^{-1/2} - \D_\lambda^{-1/2} \CCov \D_\lambda^{-1/2}\right\| \leqslant \frac{8 \mathcal{K}}{\lambda \sqrt{n}}\log (2/\delta) + \mathrm{o}\left(\frac{1}{\lambda \sqrt{n}}\right).
\end{align}
\end{proposition}
Note that the analysis behind the proof of Proposition~\ref{prop:hat_P_to_P_lambda_dimensionality_reduction} is not completely new: in \cite{pillaud2020statistical}, the convergence of the largest eigenvalue was studied using similar tools, the main difference being that in Eq.~\eqref{eq:hat_P_to_P_lambda_dimensionality_reduction}, the bound is in operator norm. Note also that for the sake of clarity, we only emphasized the inequality in the regime where $\lambda \sqrt{n}$ is large but an explicit non-asymptotic bound is given in Lemmas \ref{lemma:explicit_bound_variance} of the Appendix. 
Finally we emphasize that: (i) the bound is dimension-free, (ii) the bound is {\bfseries in operator norm} which is a \textit{strong} bound for the operator convergence as it implies many others: eigenvalue and eigenvector convergence by perturbation theory results, bound on the associated semi-group, pointwise convergence or other forms of weak convergence for operators in infinite dimension. 
\subsection{Bias analysis}
The bias analysis is harder although all objects are now deterministic. We know that $\D_\lambda^{-1/2} \CCov \D_\lambda^{-1/2}$ is a compact operator ($\CCov$ is compact and $\D_\lambda^{-1/2}$ bounded) so that its spectrum is discrete and is formed by isolated points except from $0$. On the same manner \cite[Theorem XIII.67]{reed2012methods}
the inverse of the diffusion operator $\mathcal{L}^{-1}$ is compact so that we can talk of the approximation of the $k$-th eigen-element of $\mathcal{L}^{-1}$ by the one of $\D_\lambda^{-1/2} \CCov \D_\lambda^{-1/2}$ as $\lambda$ goes to $0$ (or eigenspaces if the eigenvalues are not isolated). 

\subsubsection*{Consistency of the estimator.} First, if we are only interested in consistency of the estimator and not on rates of convergence we have the following consistency result:

\begin{proposition}[Convergence of the bias]
\label{prop:consistency_proved}
Under Assumptions~\ref{ass:spectral_gap},\ref{ass:universality},\ref{ass:smoothness},\ref{ass:boundedness}, we have the following convergence in operator norm: 
\begin{align}
\left\|\D_\lambda^{-1/2} \CCov \D_\lambda^{-1/2} - \D^{-1/2} \CCov \D^{-1/2}\right\| \underset{\lambda \to 0}{\longrightarrow} 0.
\end{align}
\end{proposition}
This results crucially relies on the fact that the operator $\D^{-1/2} \CCov \D^{-1/2}$ is compact (shown in Theorem~\ref{prop:eigen-elements_diffusion}), combined with some algebraic manipulations. 

\subsubsection*{Fast rates for smooth eigenfunctions.} 
Without more \textit{a priori knowledge} on the distribution $\mu$ (and the RKHS), it is hard to derive universal rates of convergence of the bias with respect to the regularization parameter $\lambda$. In fact, even deriving quantitative perturbation results solely on the \textit{first eigenvalue}, $ |\|\D_\lambda^{-1/2} \CCov \D_\lambda^{-1/2}\| - \|\D^{-1/2} \CCov \D^{-1/2}\||$, which corresponds to the Poincaré constant of the distribution, is known to be a difficult problem~\cite{Ane2000}. This is out of the scope of this paper. However, eigenfunctions of these elliptic operators are known to be smooth under standard assumptions on the distribution (typically, smoothness of the Gibbs potential and fast decay of $\mu$ tails ~\cite{bogachev2022fokker}). Hence, similarly to what is done in non-parametric regression, we can exploit this and quantify the difficulty of the problem~\cite{caponnetto2007optimal} by understanding how smooth (w.r.t. the RKHS) the target function is. This is what is often referred to as a \textit{source condition} in this literature~\cite{aymeric2017thesis}. Here, for most of the applications~\cite{COIFMAN20065}, we want to approximate the $p$-eigen-elements corresponding to the largest eigenvalues of $\mathcal{L}^{-1}$ (smallest eigenvalues of $\mathcal{L}$) for some $p\in \N^*$. Let us make here a natural assumption on their smoothness.
\begin{assumption}[Regularity of the problem]
\label{ass:regularity_problem}
The $p$ first eigenvectors of $\mathcal{L}$ belongs to $\h$.
\end{assumption}
An prototypical example of when it happens for any $p \in \N^*$ is if we consider a distribution with compact support $\Omega$, density $\mu = e^{-V} \in \mathcal{C}^{\infty}(\Omega)$, and the Gaussian kernel.

Let us denote $\Pi^p : \h \to \h$ the spectral projector onto the span of the $p$ largest eigenvectors of $\D^{-1/2} \CCov \D^{-1/2}$.  Technically speaking, Assumption~\ref{ass:regularity_problem} implies that for all $v \in \mathrm{span}\, \Pi^p$,
\begin{align*}
\|\D^{-1/2} v\|_{\h} < +\infty,
\end{align*}
or equivalently in terms of operators: $\|\Pi^p\D^{-1/2}\Pi^p \|<\infty$. Indeed, thanks to Theorem~\ref{prop:eigen-elements_diffusion}-(iii), for such a $v \in \h$, $S\D^{-1/2} v \in \h$, so that, by isometry this means that $\|(S S^*)^{-1/2} S \D^{-1/2} v \|_{L^2} < +\infty$, which is equivalent to the conditions above.
\begin{proposition}[Fast rates under source condition]
\label{prop:bias_fast_rates}
Under Assumptions~\ref{ass:spectral_gap},\ref{ass:universality},\ref{ass:smoothness},\ref{ass:boundedness},\ref{ass:regularity_problem}, we have the following bound in operator norm: 
\begin{align}
\left\|\Pi^p \left(\D_\lambda^{-1/2} \CCov \D_\lambda^{-1/2} - \D^{-1/2} \CCov \D^{-1/2}\right) \Pi^p\textbf{\textbf{\textbf{}}}\right\| \leq 2 \lambda \Poinca  \|\Pi^p\D^{-1/2}\Pi^p \|.
\end{align}
\end{proposition}
This proposition means that, under the source condition, the bias in the first $p$ eigenvectors depends linearly on $\lambda$. Furthermore, this is remarkable that the two crucial regularity assumptions appear in this bound: (i) the measure of complexity of the measure $\mu$, through its Poincaré constant $\Poinca$, (ii) the smooth \textit{a priori} on the target eigenvectors we want to approximate. Here we decided to showcase, for the sake of clarity, the case where the target belongs to the RKHS, but remark that refined bounds could be easily adapted from this results under more precise (and technical) source assumptions.

\subsection{Consistency and convergence rates under source assumption}
To summarize the results and the discussion of the two previous sections, let us state here the overall consistency of the estimator as well a final bound on the empirical estimator with respect to the data.
\begin{theorem}[Consistency and convergence rates]
Under Assumptions~\ref{ass:spectral_gap},\ref{ass:universality},\ref{ass:smoothness},\ref{ass:boundedness}, for any $\delta \in (0,1/2)$ and any integer $n \geqslant \mathsf{K} \log \frac{1}{\delta }$, with $\mathsf{K}$ depending on $\mathcal{K}, \mathcal{K}_d$ with probability at least $1-2\delta$, take a sequence of regularizers such that $\lambda_n \to 0$ and $\lambda_n\sqrt{n} \to +\infty$, then our estimator is consistent 
\begin{align}
\widehat{\D}_{\lambda_n}^{-1/2} \widehat{\CCov} \widehat{\D}_{\lambda_n}^{-1/2} \underset{n \to \infty}{\longrightarrow} \D^{-1/2} \CCov \D^{-1/2},
\end{align}
where the convergence holds in operator norm. Furthermore, assume \ref{ass:regularity_problem}, then, if $\lambda = \mathcal{K} n^{-1/4}$,
\begin{align}
\left\|\Pi^p \left(\widehat{\D}_{\lambda_n}^{-1/2} \widehat{\CCov} \widehat{\D}_{\lambda_n}^{-1/2} - \D^{-1/2} \CCov \D^{-1/2}\right) \Pi^p \right\| \leq \frac{8 + 2  \mathcal{K} \ \Poinca  \|\Pi^p\D^{-1/2}\Pi^p \|}{n^{1/4}} + \mathrm{o}\left(n^{-1/4}\right),
\end{align}
where, $\forall p \in \N^*$, $\Pi^p$ is the spectral projector over the largest $p$ eigenvectors of $\D^{-1/2} \CCov \D^{-1/2}$.
\end{theorem}
The theorem quantifies the statistical performance of the built estimator: we emphasize that, under the smoothness Assumption~\ref{ass:regularity_problem}, the rate of convergence to any estimated eigenfunction of $\mathcal{L}$ \textit{does not depend on the dimension}. This contrasts with the $n^{-1/d}$ rates of graph Laplacian/diffusion maps. This difference echoes the more general and intrinsic difference between local approximation techniques and kernel methods, which adapt to the underlying regularity of the problem.
\section{Numerical construction of the estimator}

Beyond their statistical performance, kernel methods also enjoy good numerical strategies to reduce their computational cost while keeping their overall precision~\cite[Section 19]{martinsson2020randomized}. We discuss informally how to apply them in our context.

\subsubsection*{Computing the estimator: naive approach.} To compute the estimator $\widehat{\D}_{\lambda}^{-1/2} \widehat{\CCov} \widehat{\D}_{\lambda}^{-1/2}$, one needs to be able to represent the operators
\begin{align*}
 \Covn  = \frac{1}{n} \sum_{i=1}^n K_{x_i} \otimes K_{x_i}, \hspace*{1cm} \Dn  = \frac{1}{n} \sum_{i=1}^n \sum_{j=1}^d \partial_j K_{x_i} \otimes \partial_j K_{x_i},
\end{align*}
whose expressions are recalled here for the sake of clarity. In fact, it suffices to represent them on $\mathrm{Span}\{K_{x_i}\}_{i\leq n} + \mathrm{Span}\{\partial_j K_{x_i}\}_{i\leq n ,\,j \leq d}$. Once such matrices $(\Sigma, L) \in \R^{(n + nd) \times (n + nd)s}$ are built, an efficient way to compute the operator is by solving the {\bfseries generalized eigenvalue problem}: i.e., find all $(\psi_{k},\mu_{k}) \in \R^{n + nd} \times \R$, for $k \in \llbracket 1 , n + nd \rrbracket$ such that
\begin{align}
    \label{eq:generalized_eigenvalue_problem}
    \Sigma \psi_{k} = \mu_{k} L_{\lambda} \psi_{k},
\end{align}
then we can write thanks the eigenvalue decomposition $\widehat{\D}_{\lambda}^{-1/2} \widehat{\CCov} \widehat{\D}_{\lambda}^{-1/2} = \sum_{k = 1}^{n + nd} \mu_k  f_k \otimes f_k$, with
\begin{align*}
    \label{eq:generalized_eigenvalue_decomposition}
    f_{k} = \sum_{i = 1}^n \psi_k[i] K_{x_i} + \sum_{i = 1}^n \sum_{j = 1}^{d} \psi_k[n+ij] \partial_j K_{x_i}.
\end{align*}
Obviously, one of the bottleneck is to build the large matrices $\Sigma$ and $L$ that have approximately $n^2 d^2$ coefficients, and then solve their related generalized eigenvalue problem. Hence, this procedure becomes intractable if $n$ or $d$ is too large. Fortunately, there is a pass forward: to implement the kernel method we need only an approximation of the kernels matrices. We propose below two well-developed method used to reduce the computations. 

\subsubsection*{Nyström approximation/column subsampling.}
The idea of this method is to build low-rank approximations of $\Sigma, L$ by selecting only $p \in \N^*$ columns among them. Note that, in favorable cases, $p$ can be chosen as low as $\log(n)$ without hurting the statistical performances~\cite{rudi2015less}. Let us choose only the $p$ columns that refer to the elements $(K_{x_i})_{i \leq p}$. The algorithm below, presented for other purposes in~\cite{cabannes2021overcoming}, returns the eigenvectors we want to approximate:
\RestyleAlgo{ruled}
\begin{algorithm}
\caption{Compute the eigenvectors by Nyström method}
\label{alg:Nystrom}

\SetAlgoLined

\KwData{$(x_i)_{i\leq n}$, a kernel $k$, and a regularizer $\lambda$}

Compute $S_p = (k(x_i,x_l))_{i \leq n, l\leq p} \in \R^{n \times p}$ \;

Compute $D_p = (\partial_{1,j} k(x_i,x_l))_{(i \leq n, j \leq d),  l\leq p} \in  \R^{nd \times p}$ \;

Build $\mathsf{\Sigma}_p = S_p^\top S_p \in \R^{p \times p}$ and $\D_p = D_p^\top D_p \in \R^{p \times p}$   \;

Get $(\psi_k, \mu_k)_{k \leq p}$ the generalized eigen-elements of $(\mathsf{\Sigma}_p, \D_p + \lambda I_p)$ \;
\end{algorithm} 
\noindent At the end of Algorithm \ref{alg:Nystrom}, for $k \leq p$ the $k$-th approximated eigenvalue is $\mu_k$, and its associated eigenfunction writes $f_k(x) = \sum_{i = 1}^p {\Psi}_k[i] K(x_i,x)$, where $\Psi_k = (\D_p + \lambda I_p)^{-1/2} \psi_k$. In terms of numerical complexity, the main costs are due to building $\D_p$ in $\mathcal{O}(p^2nd)$ and finding the generalized eigen-elements in $\mathcal{O}(p^3)$. In the course of the algorithm we used the notation $\partial_{1,j} k(x_i,x_l)$ to stress that the derivative should apply to the first variable.

\subsubsection*{Random features.}
Random features~\cite{rahimi2008random} is another way to circumvent the problem by building explicitly features that approximate any translation invariant kernel $k(x,y) = k(x-y)$. More precisely, let $p \in \N^*$ be the number of random features, $(w_l)_{l\leq p}$ be random variables of $\R^d$ independently and identically distributed according to $\P(dw) = \int_{\R^d} \mathrm{e}^{-\mathrm{i} w^\top \delta} K(\delta) d \delta \, dw$ and $(b_l)_{ l\leq p}$ be independently and identically distributed according to the uniform law on $[0,2\pi]$, then the feature vector $\phi_p (x) = \sqrt{\frac{2}{p}} \left(\cos(w_1^\top x + b_1), \hdots, \cos(w_M^\top x + b_M)\right)^\top \in \R^M$ satisfies $K(x,x') \approx \langle \phi_p (x), \phi_p (x')\rangle_2 $. Therefore, random features allow to approximate $\C$ and $\D$ by $p \times p$ matrices.
\begin{algorithm}
\caption{Compute the eigenvectors with random features}
\label{alg:Fourier}

\SetAlgoLined

\KwData{$(x_i)_{i\leq n}$ and a regularizer $\lambda$}

Compute $S_p = (\cos(w_l^\top x_i + b_l))_{i \leq n, l\leq p} \in \R^{n \times p}$ \;

Compute $D^j_p = ( - w_l[j] \sin(w_l^\top x_i + b_l) )_{i \leq n, l \leq p} \in  \R^{n \times p}$, for all $j\leq d$ \;

Build $D_p = \sum_{j = 1}^d D^j_p \in \R^{n \times p}$   \;

Build $\mathsf{\Sigma}_p = S_p^\top S_p \in \R^{p \times p}$ and $\D_p = D_p^\top D_p \in \R^{p \times p}$   \;

Get $(\psi_k, \mu_k)_{k \leq p}$ the generalized eigen-elements of $(\mathsf{\Sigma}_p, \D_p + \lambda I_p)$ \;

\end{algorithm} 
\noindent At the end of Algorithm \ref{alg:Fourier}, for $k \leq p$ the $k$-th approximated eigenvalue is $\mu_k$, and its associated eigenfunction writes $f_k(x) = \langle \Psi_k, \phi_p (x)\rangle $, where $\Psi_k = (\D_p + \lambda I_p)^{-1/2} \psi_k$. Similarly as before, the main costs are due to building $\D_p$ in $\mathcal{O}(p^2nd)$ and finding the generalized eigen-elements in $\mathcal{O}(p^3)$.

\subsubsection*{Hermite polynomials.} To conclude this numerical section, and illustrate the results, we exhibit a prototypical example where the eigenfunctions of $\mathcal{L}$ are known; and we estimate them. Indeed, take $\mu(x) = e^{-x^2/2}$, the one dimensional Gaussian. Then $\mathcal{L}f = f'' - x f' $ is the \textit{Ornstein–Uhlenbeck} operator and it is known that its eigenfunctions are the Hermite polynomials~\cite{bakry2014}. We estimate the first five Hermite polynomials with our method, thanks to Algorithm \ref{alg:Nystrom}, with $n = p = 30$, regularization parameter $\lambda = 0.1$ and the Gaussian kernel. The result is displayed in Figure~1. The data points are displayed with dots and the built eigenfunctions are plotted with plain lines. Note that the approximation is only valid on $H^1(e^{-x^2/2})$, hence the estimated eigenfunctions behave poorly outside of the dataset: a striking example of this fact is the behavior of $\hat{h}_1$, that is a linear function on the dataset interval but diverges from it rapidly when there are no data.

\begin{figure}
\label{fig:hermite_polyomials}
    \centering
    \begin{minipage}{0.49\textwidth}
        \centering
        \includegraphics[width=\textwidth]{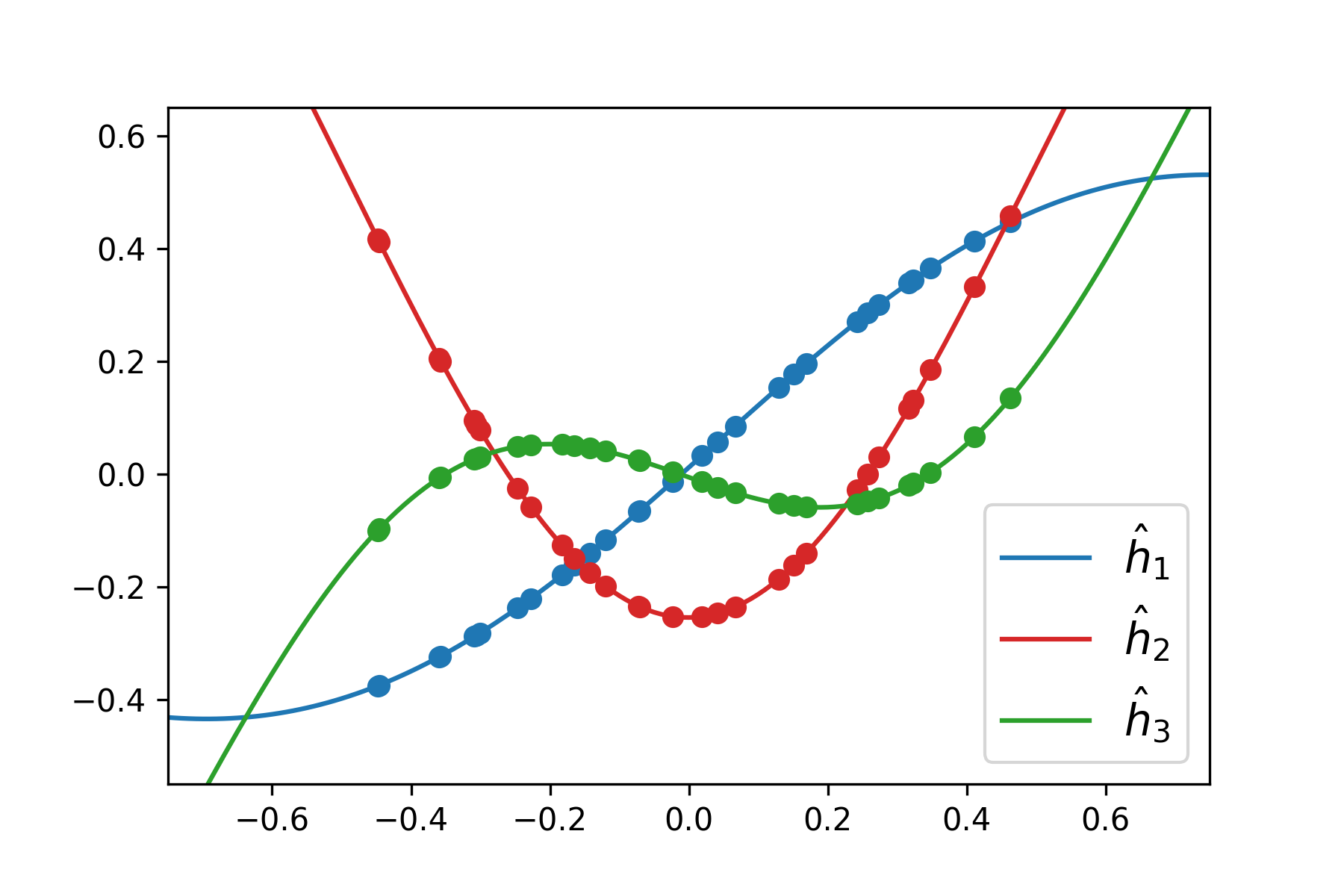} 
    \end{minipage}\hfill
    \begin{minipage}{0.49\textwidth}
        \centering
        \includegraphics[width=\textwidth]{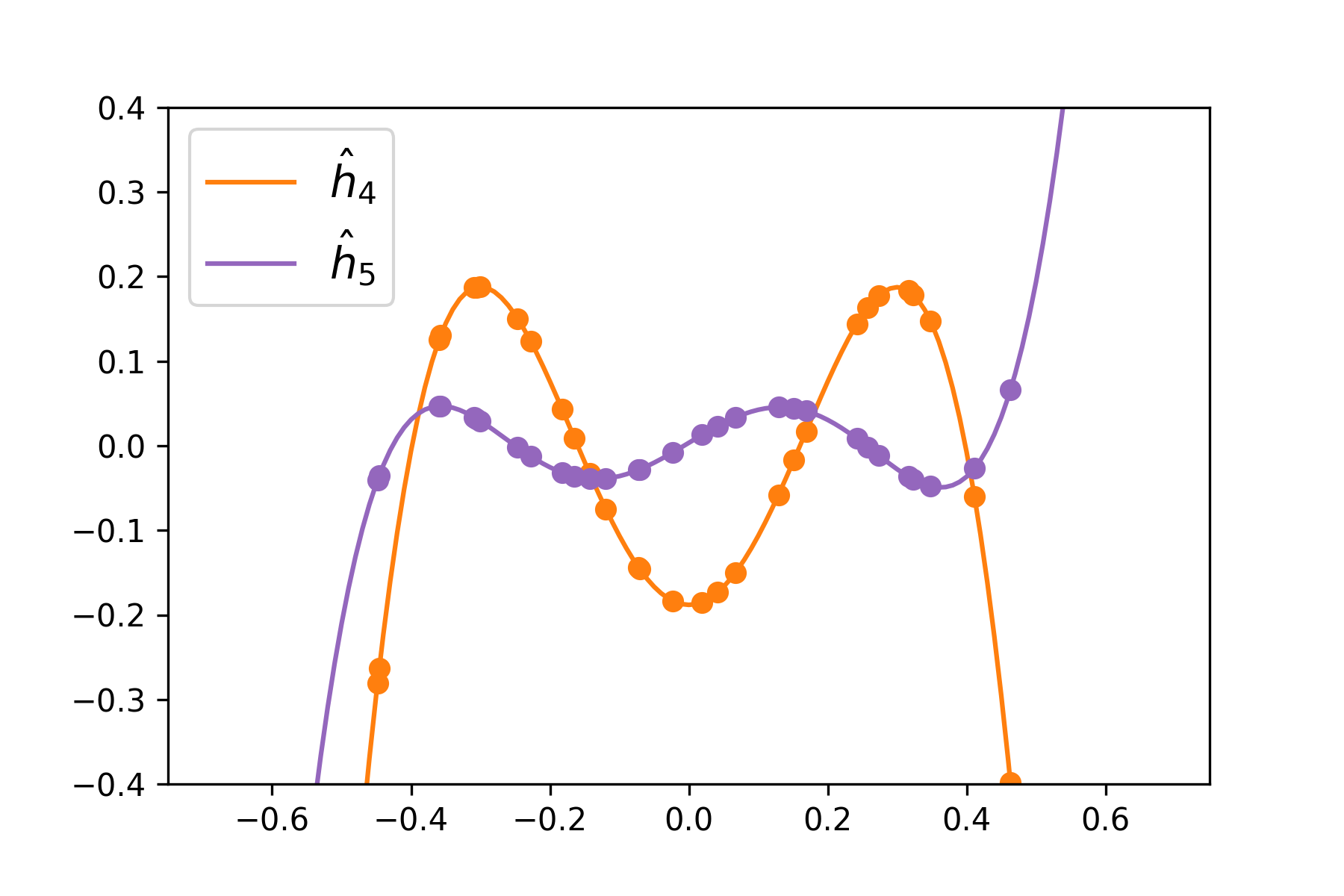} 
    \end{minipage}
    \caption{Estimation of the first five eigenfunctions of the Ornstein Uhlenbeck operator: the Hermite polynomials.}
\end{figure}

\section{Conclusion and further thoughts}

\subsubsection*{Comparison to graph Laplacians.} In this work, we proved that we could estimate the eigen-elements of the diffusion operator. This construction relies on PSD kernel methods, whereas previous rely on local averaging techniques. This leads to efficient estimations in high dimension in comparison to graph Laplacians~\cite{hein2007graph}. More precisely, under smoothness of the targeted eigenvectors, we showed that the statistical rates of the RKHS estimator does not depend on the dimension. We finally discuss computationally efficient ways to construct these eigen-elements resting on Nyström approximation~\cite{williams2000using} or Fourier Feautures~\cite{rahimi2008random}. Yet, this article focuses on the mathematical foundation of the estimator and its statistical performances: a precise computational and experimental comparison with graph Laplacians, diffusion maps and spectral clustering, as they are currently used, is left for a future work.  


\subsubsection*{The kernel choice.} Another discussion that we only sketched is the choice of the kernel. In fact, when it comes to specific applications, the art of kernel engineering is a central question. For this problem, as emphasized by the source condition, the RKHS should be chosen to approximate well the eigenfunctions of the diffusion operator in $H^1(\mu)$. Generally speaking, the interplay between $H^1(\mu)$ and $\h$ is a fundamental question at the core of RKHS approximation theory, and understanding this link should enable dimensionless approach in most of the cases. 

\subsubsection*{Markov chain.} A recent literature in applied probability aims at estimating the spectral gaps of Markov chain given the first $n$ iterates of it~\cite{hsu2015mixing}. The first eigenvalue of the estimator seems to do exactly the same, and understanding the difference between our algorithm and theirs is something worth of exploration. This will require to adapt a bit our algorithm and change our i.i.d.~assumption on the samples to a Markovian one.

\clearpage

\bibliographystyle{alpha}
\bibliography{Laplacian.bib}

\newcommand{\etalchar}[1]{$^{#1}$}
\begin{thebibliography}{CPVBR21}

\bibitem[BBG20]{berthier2020tight}
Raphaël Berthier, Francis Bach, and Pierre Gaillard.
\newblock Tight nonparametric convergence rates for stochastic gradient descent
  under the noiseless linear model, 2020.

\bibitem[BGL14]{bakry2014}
Dominique Bakry, Ivan Gentil, and Michel Ledoux.
\newblock {\em Analysis and Geometry of Markov Diffusion Operators}.
\newblock Springer, 2014.

\bibitem[BKRS22]{bogachev2022fokker}
Vladimir~I. Bogachev, Nicolai~V. Krylov, Michael R{\"o}ckner, and Stanislav~V.
  Shaposhnikov.
\newblock {\em Fokker--Planck--Kolmogorov Equations}, volume 207.
\newblock American Mathematical Society, 2022.

\bibitem[CBB22]{cabannes2022minimal}
Vivien Cabannes, Alberto Bietti, and Randall Balestriero.
\newblock On minimal variations for unsupervised representation learning.
\newblock {\em arXiv preprint arXiv:2211.03782}, 2022.

\bibitem[CBLK06]{Boaz2006}
Ronald Coifman, Nadler Boaz, St{\'e}phane Lafon, and Ioannis Kevrekidis.
\newblock Diffusion maps, spectral clustering and reaction coordinates of
  dynamical systems.
\newblock {\em Applied and Computational Harmonic Analysis}, 21(12):113--127,
  2006.

\bibitem[CDV07]{caponnetto2007optimal}
Andrea Caponnetto and Ernesto De~Vito.
\newblock Optimal rates for the regularized least-squares algorithm.
\newblock {\em Foundations of Computational Mathematics}, 7(3):331--368, 2007.

\bibitem[CL06]{COIFMAN20065}
Ronald~R. Coifman and St{\'e}phane Lafon.
\newblock Diffusion maps.
\newblock {\em Applied and Computational Harmonic Analysis}, 21(1), 2006.

\bibitem[CPVBR21]{cabannes2021overcoming}
Vivien Cabannes, Loucas Pillaud-Vivien, Francis Bach, and Alessandro Rudi.
\newblock Overcoming the curse of dimensionality with laplacian regularization
  in semi-supervised learning.
\newblock {\em Advances in Neural Information Processing Systems},
  34:30439--30451, 2021.

\bibitem[Cra76]{crawford1976stable}
Charles~R Crawford.
\newblock A stable generalized eigenvalue problem.
\newblock {\em SIAM Journal on Numerical Analysis}, 13(6):854--860, 1976.

\bibitem[CSP{\etalchar{+}}00]{Ane2000}
An{\'e} C{\'e}cile, Blach{\`e}re S{\'e}bastien, Foug{\`e}res Pierre, Gentil
  Ivan, Malrieu Florent, Roberto Cyril, and Scheffer Gr{\'e}gory.
\newblock {\em Sur les In{\'e}galit{\'e}s de {S}obolev Logarithmiques}.
\newblock Panoramas et Synth{\`e}ses 10, Soci{\'e}t{\'e} Math{\'e}matique de
  France (SMF), 2000.

\bibitem[Die17]{aymeric2017thesis}
Aymeric Dieuleuveut.
\newblock {\em Stochastic Approximation in {H}ilbert Spaces}.
\newblock PhD thesis, ENS - INRIA, 2017.

\bibitem[HAL07]{hein2007graph}
Matthias Hein, Jean-Yves Audibert, and Ulrike~von Luxburg.
\newblock Graph laplacians and their convergence on random neighborhood graphs.
\newblock {\em Journal of Machine Learning Research}, pages 1325--1368, 2007.

\bibitem[HKS15]{hsu2015mixing}
Daniel Hsu, Aryeh Kontorovich, and Csaba Szepesv{\'a}ri.
\newblock Mixing time estimation in reversible markov chains from a single
  sample path.
\newblock In {\em Advances in Neural Information Processing Systems}, pages
  1459--1467, 2015.

\bibitem[HN05]{Helffer2005witten}
Bernard Helffer and Francis Nier.
\newblock Hypoelliptic estimates and spectral theory for {F}okker-{P}lanck
  operators and {W}itten {L}aplacians.
\newblock {\em Lecture Notes in Mathematics}, 1862, 2005.

\bibitem[Hot33]{hotelling1933analysis}
Harold Hotelling.
\newblock Analysis of a complex of statistical variables into principal
  components.
\newblock {\em Journal of Educational Psychology}, 24(6):417, 1933.

\bibitem[KNP{\etalchar{+}}20]{klus2020data}
Stefan Klus, Feliks N{\"u}ske, Sebastian Peitz, Jan-Hendrik Niemann, Cecilia
  Clementi, and Christof Sch{\"u}tte.
\newblock Data-driven approximation of the koopman generator: Model reduction,
  system identification, and control.
\newblock {\em Physica D: Nonlinear Phenomena}, 406:132416, 2020.

\bibitem[Lel13]{Lelievre2013}
Tony Leli{\`e}vre.
\newblock Two mathematical tools to analyze metastable stochastic processes.
\newblock In {\em Numerical Mathematics and Advanced Applications 2011}, pages
  791--810, Berlin, Heidelberg, 2013. Springer.

\bibitem[MCRR20]{meanti2020kernel}
Giacomo Meanti, Luigi Carratino, Lorenzo Rosasco, and Alessandro Rudi.
\newblock Kernel methods through the roof: handling billions of points
  efficiently.
\newblock {\em Advances in Neural Information Processing Systems}, 2020.

\bibitem[MSS{\etalchar{+}}99]{mika1999kernel}
Sebastian Mika, Bernhard Sch{\"o}lkopf, Alex~J. Smola, Klaus-Robert M{\"u}ller,
  Matthias Scholz, and Gunnar R{\"a}tsch.
\newblock Kernel {P}{C}{A} and de-noising in feature spaces.
\newblock In {\em Advances in Neural Information Processing Systems}, pages
  536--542, 1999.

\bibitem[MT20]{martinsson2020randomized}
Per-Gunnar Martinsson and Joel~A. Tropp.
\newblock Randomized numerical linear algebra: Foundations and algorithms.
\newblock {\em Acta Numerica}, 29:403--572, 2020.

\bibitem[MXZ06]{micchelli2006universal}
Charles~A Micchelli, Yuesheng Xu, and Haizhang Zhang.
\newblock Universal kernels.
\newblock {\em Journal of Machine Learning Research}, 7:2651--2667, 2006.

\bibitem[Nad64]{nadaraya1964estimating}
Elizbar~A. Nadaraya.
\newblock On estimating regression.
\newblock {\em Theory of Probability \& Its Applications}, 9(1):141--142, 1964.

\bibitem[NJW01]{ng2001spectral}
Andrew Ng, Michael Jordan, and Yair Weiss.
\newblock On spectral clustering: Analysis and an algorithm.
\newblock {\em Advances in neural information processing systems}, 14, 2001.

\bibitem[Pea01]{pearson1901liii}
Karl Pearson.
\newblock On lines and planes of closest fit to systems of points in space.
\newblock {\em The London, Edinburgh, and Dublin Philosophical Magazine and
  Journal of Science}, 2(11):559--572, 1901.

\bibitem[PV20]{pillaud2020learning}
Loucas Pillaud-Vivien.
\newblock {\em Learning with Reproducing Kernel Hilbert Spaces: Stochastic
  Gradient Descent and Laplacian Estimation}.
\newblock PhD thesis, Universit{\'e} Paris sciences et lettres, 2020.

\bibitem[PVBL{\etalchar{+}}20]{pillaud2020statistical}
Loucas Pillaud-Vivien, Francis Bach, Tony Leli{\`e}vre, Alessandro Rudi, and
  Gabriel Stoltz.
\newblock Statistical estimation of the {P}oincar{\'e} constant and application
  to sampling multimodal distributions.
\newblock In {\em International Conference on Artificial Intelligence and
  Statistics}, pages 2753--2763, 2020.

\bibitem[PVRB18]{pillaud2018statistical}
Loucas Pillaud-Vivien, Alessandro Rudi, and Francis Bach.
\newblock Statistical optimality of stochastic gradient descent on hard
  learning problems through multiple passes.
\newblock {\em Advances in Neural Information Processing Systems}, 31, 2018.

\bibitem[RCR15]{rudi2015less}
Alessandro Rudi, Raffaello Camoriano, and Lorenzo Rosasco.
\newblock Less is more: Nystr{\"o}m computational regularization.
\newblock In {\em Advances in Neural Information Processing Systems}, pages
  1657--1665, 2015.

\bibitem[RR08]{rahimi2008random}
Ali Rahimi and Benjamin Recht.
\newblock Random features for large-scale kernel machines.
\newblock In {\em Advances in neural information processing systems}, pages
  1177--1184, 2008.

\bibitem[RR17]{rudi2017generalization}
Alessandro Rudi and Lorenzo Rosasco.
\newblock Generalization properties of learning with random features.
\newblock In {\em Advances in Neural Information Processing Systems}, pages
  3218--3228, 2017.

\bibitem[RS12]{reed2012methods}
Michael Reed and Barry Simon.
\newblock {\em Methods of Modern Mathematical Physics: Functional Analysis},
  volume~IV.
\newblock Elsevier, 2012.

\bibitem[Sal98]{salinelli1998}
Ernesto Salinelli.
\newblock Nonlinear principal components i. absolutely continuous random
  variables with positive bounded densities.
\newblock {\em Annals of Statistics}, 26(2):596--616, 1998.

\bibitem[SC08]{steinwart2008support}
Ingo Steinwart and Andreas Christmann.
\newblock {\em Support Vector Machines}.
\newblock Springer Science \& Business Media, 2008.

\bibitem[SS02]{smola-book}
B.~Sch{\"o}lkopf and A.~J. Smola.
\newblock {\em Learning with Kernels}.
\newblock MIT Press, 2002.

\bibitem[Sta17]{Minsker2011}
Minsker Stanislav.
\newblock On some extensions of {B}ernstein's inequality for self-adjoint
  operators.
\newblock {\em Statistics and Probability Letters}, 127:111--119, 2017.

\bibitem[SW06]{schaback2006kernel}
Robert Schaback and Holger Wendland.
\newblock Kernel techniques: from machine learning to meshless methods.
\newblock {\em Acta Numerica}, 15:543, 2006.

\bibitem[Tro12]{Tropp2012UserFriendlyTF}
Joel~A. Tropp.
\newblock User-friendly tools for random matrices: an introduction.
\newblock {\em NIPS Tutorials}, 2012.

\bibitem[VL07]{von2007tutorial}
Ulrike Von~Luxburg.
\newblock A tutorial on spectral clustering.
\newblock {\em Statistics and computing}, 17(4):395--416, 2007.

\bibitem[VPVF21]{varre2021last}
Aditya~Vardhan Varre, Loucas Pillaud-Vivien, and Nicolas Flammarion.
\newblock Last iterate convergence of sgd for least-squares in the
  interpolation regime.
\newblock {\em Advances in Neural Information Processing Systems},
  34:21581--21591, 2021.

\bibitem[WS00]{williams2000using}
Christopher Williams and Matthias Seeger.
\newblock Using the nystr{\"o}m method to speed up kernel machines.
\newblock {\em Advances in neural information processing systems}, 13, 2000.

\bibitem[Yur95]{Yurinsky1995}
Vadim~Vladimirovich Yurinsky.
\newblock {\em Gaussian and {R}elated {A}pproximations for {D}istributions of
  {S}ums}, pages 163--216.
\newblock Springer Berlin Heidelberg, 1995.

\bibitem[ZB05]{zwald2005convergence}
Laurent Zwald and Gilles Blanchard.
\newblock On the convergence of eigenspaces in kernel principal component
  analysis.
\newblock {\em Advances in Neural Information Processing Systems}, 18, 2005.

\bibitem[Zho08]{Zhou2008}
Ding-Xuan Zhou.
\newblock Derivative reproducing properties for kernel methods in learning
  theory.
\newblock {\em Journal of Computational and Applied Mathematics},
  220(1):456--463, 2008.

\end{thebibliography}

\clearpage

\appendix

{\bfseries \Huge \noindent Appendix}

\vspace{1cm}

\noindent Note that is all the appendix, to avoid cumbersome notations we will make no distinctions between $\h_0$ and $\h$ and $H^1_0(\mu)$ and $H^1(\mu)$ unless it is strictly necessary.  

\section{Proof on the embedding of the diffusion operator and its inverse}

\subsection{Proof of Proposition~\ref{prop:representation_diffusion}}

We want to show that on $H^1_0(\mu)$, $\mathcal{L}^{-1} = \mathsf{S} \D^{-1} \mathsf{S}^*$. \\

\begin{proof}[Proof of Proposition~\ref{prop:representation_diffusion}]
Let $g \in \ran\ \mathsf{S}$, there exists $f \in \h$ such that $ g = \mathsf{S} f$. Let us calculate:
$$ \mathsf{S} \D^{-1} \mathsf{S}^* \mathcal{L} g = \mathsf{S} \D^{-1} \mathsf{S}^* \mathcal{L} \mathsf{S} f =  \mathsf{S} \D^{-1} \D f = \mathsf{S} f = g. $$
Moreover, as $\mathcal{L}$ is invertible on $\ran\ \mathsf{S} \cap H^1(\mu) \cap L_0^2(\mu)$, the left and right inverse are the same. Hence, $\mathcal{L}^{-1}$ and $\mathsf{S} \D^{-1} \mathsf{S}^*$ are equal on $\ran\ \mathsf{S}$.
Furthermore we can notice that $\mathcal{L}^{-1}$ and $\mathsf{S} \D^{-1} \mathsf{S}^*$ are bounded on $L^2(\mu)$. Indeed, 
\begin{align*}
\Poinca =\!\!\!\! \sup_{f \in (\ker \D)^\perp}\!\!\!\! \frac{\langle f, \mathsf{S}^*\mathsf{S} f\rangle_\h }{\langle f, \D f\rangle_\h} \geqslant \!\!\!\! \sup_{f \in (\ker \D)^\perp} \!\!\!\! \frac{\langle \D^{-1/2}f, \mathsf{S}^*\mathsf{S} \D^{-1/2} f\rangle_\h }{\langle \D^{-1/2}f, \D \D^{-1/2} f\rangle_\h} &=\!\!\!\! \sup_{f \in (\ker \D)^\perp}\!\!\!\! \frac{\langle f, \D^{-1/2} \mathsf{S}^*\mathsf{S} \D^{-1/2} f\rangle_\h }{\|f\|^2_\h} \\
&= \|\D^{-1/2} \mathsf{S}^*\mathsf{S} \D^{-1/2}\|_\h \\
&= \|\mathsf{S} \D^{-1}\mathsf{S}^*\|_{L^2(\mu)}.
\end{align*}
As $\mathcal{L}^{-1}$ and $\mathsf{S} \D^{-1} \mathsf{S}^*$ are equal and continuous on $\ran\ \mathsf{S}$, they are also equal on its closure.
\end{proof}

\subsection{Proof of Theorem~\ref{prop:eigen-elements_diffusion} through a technical result on operators between Hilbert spaces}

The lemma below gives the proof of  Theorem~\ref{prop:eigen-elements_diffusion} considering $A = S \D^{-1/2}$, $\mathcal{H}_1 = \h_0$ and $\mathcal{H}_2 = H^1_0(\mu)$.

\begin{lemma}[Link between $\mathsf{A}^*\mathsf{A}$ and  $\mathsf{A}\mathsf{A}^*$ in the compact case.]
\label{le:reversement_operator}
Let $\h_1$ and $\h_2$ two Hilbert spaces. Let $A$ be an operator from $\h_1$ to $\h_2$ such that $\mathsf{A}^*\mathsf{A}$ is a self-adjoint compact operator on $\h_1$. Then,
\begin{enumerate}[label=(\roman*), itemsep = 0pt]
\item $\mathsf{A}$ is a bounded operator from $\h_1$ to  $\h_2$.
\item $\mathsf{A}\mathsf{A}^*$ is a self-adjoint compact operator on $\h_2$ with the same spectrum as $\mathsf{A}\mathsf{A}^*$.
\item If $\lambda \neq 0$ is an eigenvalue of $\mathsf{A}^*\mathsf{A}$ with eigenvector $u \in \h_1$, then $\lambda$ is an eigenvalue of $\mathsf{A}\mathsf{A}^*$ with eigenvector $\mathsf{A}u \in \h_2$.
\end{enumerate}
\end{lemma}
\begin{proof}
First let us notice that $\mathsf{A}$ is necessarily bounded. Indeed, let $u \in \h_1$, %
\begin{align*}
 \left\|\mathsf{A} u\right\|_{\h_2}^2 = \left\langle{\mathsf{A} u, \mathsf{A} u}\right\rangle_{\h_2} = \left\langle{\mathsf{A}^*\mathsf{A} u, u}\right\rangle_{\h_1}  \leqslant \left\|\mathsf{A}^*\mathsf{A} u\right\|_{\h_1} \left\|u\right\|_{\h_1} \leqslant \left\|\mathsf{A}^*\mathsf{A} \right\| \left\|u\right\|_{\h_1}^2.
\end{align*}%
Hence, $\|\mathsf{A}\| \leqslant \sqrt{ \left\|\mathsf{A}^*\mathsf{A} \right\|}$.
%
%

Second, as $\mathsf{A}^*\mathsf{A}$ is self-adjoint and compact on $\h_1$, there exists $(\psi_i)_{i \in \mathbb{N}}$ an orthonormal basis~$\h_1$ and  a sequence of reals $(\lambda_i)_{i \in \mathbb{N}}$ such that:
$$ \mathsf{A}^*\mathsf{A} = \sum_{i \geqslant 0} \lambda_i \psi_i \otimes \psi_i, $$
where the infinite sum stands for the strong convergence of operators. Now, by composing on the left side by $\mathsf{A}^*$ and on the right side by $\mathsf{A}^*$, we get:
$$ (\mathsf{A}\mathsf{A}^*)^2 = \mathsf{A}\mathsf{A}^*\mathsf{A}\mathsf{A}^* = \sum_{i \geqslant 0} \lambda_i (\mathsf{A} \psi_i) \otimes (\mathsf{A}\psi_i) = \sum_{i \geqslant 0} \lambda^2_i (\mathsf{A} \frac{\psi_i}{\sqrt{\lambda_i}}) \otimes (\mathsf{A}\frac{\psi_i}{\sqrt{\lambda_i}}). $$
Hence, $\mathsf{A}\mathsf{A}^* = \sum_{i \geqslant 0} \lambda_i (\mathsf{A} \frac{\psi_i}{\sqrt{\lambda_i}}) \otimes (\mathsf{A}\frac{\psi_i}{\sqrt{\lambda_i}})$ and is a compact operator. We can of course check that $\left(\mathsf{A} \frac{\psi_i}{\sqrt{\lambda_i}}\right)_{i \in \mathbb{N}}$ is an orthonormal basis of $\h_2$: $\left \langle \mathsf{A} \frac{\psi_i}{\sqrt{\lambda_i}}, \mathsf{A} \frac{\psi_j}{\sqrt{\lambda_j}}\right\rangle = (\lambda_i \lambda_j)^{-1/2} \left \langle \psi_i, \mathsf{A}^*\mathsf{A}\psi_j\right\rangle = \sqrt{\frac{\lambda_j}{\lambda_i}} \left \langle \psi_i,\psi_j\right\rangle = \delta_{ij}$.
\end{proof}

\section{Bound on the variance term}
\label{sec:technical_inequalities}

We begin first to recall usual concentration inequalities that will help handling the variance term.

\subsection{Concentration inequalities}

We first begin by recalling some concentration inequalities for sums of random vectors and operators.

\begin{proposition}[Bernstein’s inequality for sums of random vectors]
\label{prop:Bernstein_vector}
Let $z_1,\hdots, z_n$ be a sequence of independent identically and distributed random elements of a separable Hilbert space
$\h$. Assume that $\E \|z_1\|<+\infty$ and note $\mu = \E z_1$. Let $\sigma, L \geqslant 0$ such that, $$ \forall p \geqslant 2, \qquad\E \left\|z_1-\mu\right\| ^p_\h \leqslant \frac{1}{2} p!\sigma^2L^{p-2}.$$ 
Then, for any $\delta \in (0,1]$,
\begin{align}
\left\|\frac{1}{n}\sum_{i=1}^n z_i - \mu\right\|_\h\leqslant \frac{2L \log(2/\delta)}{n} + \sqrt{\frac{2 \sigma^2 \log(2/\delta)}{n}},
\end{align}
with probability at least $1-\delta$.
\end{proposition}

\begin{proof}
This is a restatement of Theorem 3.3.4 of \cite{Yurinsky1995}.
\end{proof}

\begin{proposition}[Bernstein’s inequality for sums of random operators]
\label{prop:Bernstein_operator}
Let $\h$ be a separable Hilbert space and let $X_1, \hdots, X_n$ be a sequence of independent and identically distributed
self-adjoint random operators on $\h$. Assume that  $\E(X_i) = 0$ and that there exist $T > 0$ and $S$ a
positive trace-class operator such that $\|X_i\| \leqslant T$ almost surely and $\E X_i^2 \preccurlyeq S$ for any $i \in \{1, \hdots, n\}$. Then, for any $\delta \in (0,1]$, the following inequality holds:
\begin{align}
\label{eq:concentration_operator}
\left\|\frac{1}{n}\sum_{i=1}^n X_i\right\| \leqslant \frac{2T\beta}{3n} + \sqrt{\frac{2 \|S\|\beta}{n}},
\end{align}
with probability at least $1-\delta$ and where $\beta = \log \frac{2 \mathrm{ Tr } S}{\|S\|\delta}$.
\end{proposition}

\begin{proof}
The theorem is a restatement of Theorem 7.3.1 of \cite{Tropp2012UserFriendlyTF} generalized to the separable
Hilbert space case by means of the technique in Section 4 of \cite{Minsker2011}.
\end{proof}

\subsection{Operator bounds}
\label{subsec:operators}

\begin{lemma} 
Under Assumptions \ref{ass:smoothness} and \ref{ass:boundedness}, $\Cov$, and $\D$ are trace-class operators.
\end{lemma}
\begin{proof}
We only prove the result for $\D$, the proof for $\Cov$ being similar. Consider an orthonormal basis $(\phi_i)_{i\in \mathbb{N}}$ of $\h$. Then, as $\D$ is a positive self adjoint operator,
\begin{align*}
\tr\ \D &= \sum_{i = 1}^\infty \langle \D \phi_i,\phi_i \rangle = \sum_{i = 1}^\infty \E_{\mu} \left[\sum_{j=1}^d\langle \partial_j K_x,\phi_i \rangle^2\right] = \E_{\mu} \left[\sum_{i = 1}^\infty \sum_{j=1}^d\langle \partial_j K_x,\phi_i \rangle^2\right] \\
&= \E_{\mu} \left[\sum_{j=1}^d\left\|\partial_j K_x\right\|^2\right] \leqslant \mathcal{K}_d.
\end{align*}
Hence, $\D$ is a trace-class operator.
\end{proof}

\subsection{Proof of Proposition~\texorpdfstring{\ref{prop:hat_P_to_P_lambda_dimensionality_reduction}}{d}: the variance bound}

We recall here the expression of the variance of our estimator we want to control:
\begin{align*}
    \text{Variance} &:= \left\|\widehat{\D}_\lambda^{-1/2} \CCovn \widehat{\D}_\lambda^{-1/2} - \D_\lambda^{-1/2} \CCov \D_\lambda^{-1/2}\right\| \\
    & = \left\|\widehat{\D}_\lambda^{-1/2} \widehat{\CCov} \widehat{\D}_\lambda^{-1/2} - \widehat{\D}_\lambda^{-1/2} \CCov \widehat{\D}_\lambda^{-1/2}\right\| + \left \|\widehat{\D}_\lambda^{-1/2} \CCov \widehat{\D}_\lambda^{-1/2} - \D_\lambda^{-1/2} \CCov \D_\lambda^{-1/2}\right\| \\
    &\leqslant \left\|\widehat{\D}_\lambda^{-1/2} (\widehat{\CCov}-\CCov) \widehat{\D}_\lambda^{-1/2}\right\| + \left \|\widehat{\D}_\lambda^{-1/2} \CCov \widehat{\D}_\lambda^{-1/2} - \D_\lambda^{-1/2} \CCov \widehat{\D}_\lambda^{-1/2}\right\| + \left \|\D_\lambda^{-1/2} \CCov \widehat{\D}_\lambda^{-1/2} - \D_\lambda^{-1/2} \CCov \D_\lambda^{-1/2}\right\| \\
    &\leqslant \underbrace{\left\|\widehat{\D}_\lambda^{-1/2} (\widehat{\CCov}-\CCov) \widehat{\D}_\lambda^{-1/2}\right\|}_{\text{Lemma~\ref{lemma:concentration_C}}} + \underbrace{\left \|(\widehat{\D}_\lambda^{-1/2} - \D_\lambda^{-1/2}) \CCov \widehat{\D}_\lambda^{-1/2} \right\|}_{\text{Lemma~\ref{lemma:concentration_left}}} +  \underbrace{\left \|\D_\lambda^{-1/2} \CCov (\widehat{\D}_\lambda^{-1/2} -  \D_\lambda^{-1/2}) \right\|}_{\text{Lemma~\ref{lemma:concentration_left} as well}}. 
\end{align*}
The following quantities are useful for the estimates in this section:
\begin{align*}
\Nla = \sup_{x \in \mathrm{supp} (\mu) } \left\| \D_{\lambda}^{-1/2} K_x \right\|^2_\h, \ \textrm{and} \quad \Fla = \sup_{x \in \mathrm{supp} (\mu) } \left\| \D_{\lambda}^{-1/2} \nabla K_x \right\|^2_\h.
\end{align*}
Note that under Assumption~\ref{ass:boundedness}, $\Nla \leqslant \frac{\mathcal{K}}{\lambda}$ and $\Fla \leqslant \frac{\mathcal{K}_d}{\lambda}$. Note also that under refined assumptions on the spectrum of $\D$, we could have a better dependence of the latter bounds with respect to $\lambda$. We first state the overall result before showing all the auxiliary lemmas below:
\begin{lemma}
\label{lemma:explicit_bound_variance}
For any $ 0< \lambda < \|\D\|$, $n \geqslant 15 \Fla \log \frac{4\,\mathrm{Tr } \D }{\lambda \delta }$ and any $\delta \in (0,1/2]$,
\begin{align*}
\left\|\widehat{\D}_\lambda^{-1/2} \CCovn \widehat{\D}_\lambda^{-1/2} - \D_\lambda^{-1/2} \CCov \D_\lambda^{-1/2}\right\| &\leqslant \frac{4 \Nla \log \frac{2\,\Poinca\mathrm{Tr } \CCov }{ \lambda \delta } }{3n} + \left[\frac{ 2\ \Poinca\ \Nla \log \frac{4\,\Poinca\mathrm{Tr } \CCov }{ \lambda \delta }}{n}\right]^{1/2} \\ 
&\hspace*{0.5cm}+ 4 \left(\Poinca \left\|\CCov\right\|\right)^{1/2}  \left(\frac{4 \Fla \log \frac{4\,\mathrm{Tr } \D }{\lambda \delta }}{3n} + \sqrt{\frac{ 2\ \Fla \log \frac{4\,\mathrm{Tr } \D }{\lambda \delta }}{n}} \right),
\end{align*}
with probability at least $1-2\delta$.
\end{lemma}

\subsubsection{Bound on the first term}

\begin{lemma}
\label{lemma:concentration_C}
For any $ \lambda > 0$, and any $\delta \in (0,1]$,
\begin{align*}
\left\|\D_\lambda^{-1/2}  (\widehat{\CCov}-\CCov) \D_\lambda^{-1/2}\right\|&\leqslant \frac{4 \Nla \log \frac{2\,\Poinca\mathrm{Tr } \Cov }{ \lambda \delta } }{3n} + \left[\frac{ 2\ \Poinca\ \Nla \log \frac{4\,\Poinca\mathrm{Tr } \Cov }{ \lambda \delta }}{n}\right]^{1/2},
\end{align*}
with probability at least $1-\delta$.
\end{lemma}

\begin{proof}[Proof of Lemma \ref{lemma:concentration_C}]
We apply some concentration inequality to the operator $\D_\lambda^{-1/2} \widehat{\CCov}\D_\lambda^{-1/2}$ whose mean is exactly $\D_\lambda^{-1/2} \CCov\D_\lambda^{-1/2}$. The calculation is the following:
\begin{align*}
\left\|\D_\lambda^{-1/2}  (\widehat{\CCov}-\CCov) \D_\lambda^{-1/2}\right\| &= \left\|\D_\lambda^{-1/2} \widehat{\CCov} \D_\lambda^{-1/2} - \D_\lambda^{-1/2} \CCov \D_\lambda^{-1/2} \right\| \\  
&= \left\| \frac{1}{n}\sum_{i=1}^n \left[ (\D_\lambda^{-1/2} K_{x_i})\otimes (\D_\lambda^{-1/2} K_{x_i}) - \D_\lambda^{-1/2} \Cov \D_\lambda^{-1/2} \right] \right\|.
\end{align*}
We use Proposition \ref{prop:Bernstein_operator}. To do this, we bound for $i \in \llbracket 1,n\rrbracket$ :
\begin{align*}
\left\|(\D_\lambda^{-1/2} K_{x_i})\otimes (\D_\lambda^{-1/2} K_{x_i}) - \D_\lambda^{-1/2} \Cov \D_\lambda^{-1/2}\right\| &\leqslant \left\| \D_{\lambda}^{-1/2} K_{x_i} \right\|^2_\h + \left\|\D_\lambda^{-1/2} \C \D_\lambda^{-1/2}\right\|\\
&\leqslant 2\Nla,
\end{align*}
and, for the second order moment,
\begin{align*}
\E &\left((\D_\lambda^{-1/2} K_{x_i})\otimes (\D_\lambda^{-1/2} K_{x_i}) - \D_\lambda^{-1/2} \Cov \D_\lambda^{-1/2}\right)^2 \\
&\hspace*{1cm}= \E \left[\left\| \D_{\lambda}^{-1/2} K_{x_i} \right\|^2_\h  (\D_\lambda^{-1/2} K_{x_i})\otimes (\D_\lambda^{-1/2} K_{x_i})\right] - \D_\lambda^{-1/2} \Cov \D_\lambda^{-1} \Cov \D_\lambda^{-1/2} \\
&\hspace*{1cm}\preccurlyeq \Nla \D_\lambda^{-1/2} \Cov \D_\lambda^{-1/2}.
\end{align*}
We conclude the proof by some estimation of the constant $\beta = \log \frac{2\,\mathrm{Tr } (\Cov \D_\lambda^{-1})}{\left\|\D_\lambda^{-1/2} \Cov \D_\lambda^{-1/2}\right\| \delta }$. To do this, we remark that, thanks to Proposition 16 of \cite{pillaud2020learning},
%
\begin{align*}
\label{eq:defaut_poinca}
  \left\|\D_\lambda^{-1/2} \Cov \D_\lambda^{-1/2}\right\| = (\Poinca^\lambda_\mu)^{-1} \geqslant \Poinca^{-1},
\end{align*}
and using $\mathrm{Tr } \C \D_\lambda^{-1} \leqslant \lambda^{-1} \mathrm{Tr } \Cov $, it holds $\beta \leqslant \log \frac{2\,\Poinca\mathrm{Tr } \Cov }{\lambda \delta }$. Therefore,
\begin{align*}
&\left\| \frac{1}{n}\sum_{i=1}^n \left[ (\D_\lambda^{-1/2} K_{x_i})\otimes (\D_\lambda^{-1/2} K_{x_i}) - \D_\lambda^{-1/2} \C \D_\lambda^{-1/2} \right] \right\| \\
&\hspace*{3cm}\leqslant \frac{4 \Nla \log \frac{2\,\Poinca\mathrm{Tr } \C }{\lambda \delta } }{3n} + \left[\frac{ 2\ {\Poinca^\lambda_\mu}\ \Nla \log \frac{2\, \Poinca\mathrm{Tr } \C }{ \lambda \delta }}{n}\right]^{1/2}.
\end{align*}
This concludes the proof of Lemma~\ref{lemma:concentration_C}.
\end{proof}

\subsubsection{Bound on the second term}

Here, we want to bound the term $\left \|(\widehat{\D}_\lambda^{-1/2} - \D_\lambda^{-1/2}) \CCov \widehat{\D}_\lambda^{-1/2} \right\|$. Let us work on it a little bit more.
\begin{align*}
    \left \|(\widehat{\D}_\lambda^{-1/2} - \D_\lambda^{-1/2}) \CCov \widehat{\D}_\lambda^{-1/2} \right\| &= \left \|\widehat{\D}_\lambda^{1/2}(\widehat{\D}_\lambda^{-1} - \D_\lambda^{-1}) \D_\lambda^{1/2} \CCov \widehat{\D}_\lambda^{-1/2} \right\| \\
    &= \left \|\widehat{\D}_\lambda^{-1/2}(\widehat{\D}_\lambda - \D_\lambda) \D_\lambda^{-1/2} \CCov \widehat{\D}_\lambda^{-1/2} \right\|  \\
    &= \left \|\widehat{\D}_\lambda^{-1/2}\D_\lambda^{1/2}\D_\lambda^{-1/2}(\widehat{\D}_\lambda - \D_\lambda) \D_\lambda^{-1/2} \CCov^{1/2} \CCov^{1/2} \D_\lambda^{-1/2}\D_\lambda^{1/2}\widehat{\D}_\lambda^{-1/2} \right\|  \\
    &\leq  \left \|\widehat{\D}_\lambda^{-1/2}\D_\lambda^{1/2} \right\| \left\|\D_\lambda^{-1/2}(\widehat{\D}_\lambda - \D_\lambda) \D_\lambda^{-1/2} \right\| \left\|\CCov^{1/2}\right\| \left\|\CCov^{1/2} \D_\lambda^{-1/2} \right\|  \left \|\widehat{\D}_\lambda^{-1/2}\D_\lambda^{1/2} \right\| \\
    &\leq  \left(\Poinca \left\|\CCov\right\|\right)^{1/2} \underbrace{\left\|\widehat{\D}_\lambda^{-1/2}\D_\lambda^{1/2} \right\|^2}_{\text{Lemma~\ref{lemma:magic_2}}} \underbrace{\left\|\D_\lambda^{-1/2}(\widehat{\D}_\lambda - \D_\lambda) \D_\lambda^{-1/2} \right\|}_{\text{Lemma~\ref{lemma:concentration_delta}}} . 
\end{align*}
Hence, we can formulate the principal result of this subsection:
\begin{lemma}
\label{lemma:concentration_left}
For any $0<\lambda < \|\D\| $, $\delta \in (0,1)$, and $n \geqslant 15 \Fla \log \frac{4\,\mathrm{Tr } \D }{\lambda \delta }$, it holds with probability at least $1-\delta$:
\begin{align*}
    \left \|(\widehat{\D}_\lambda^{-1/2} - \D_\lambda^{-1/2}) \CCov \widehat{\D}_\lambda^{-1/2} \right\| \leq 2 \left(\Poinca \left\|\CCov\right\|\right)^{1/2}  \left(\frac{4 \Fla \log \frac{4\,\mathrm{Tr } \D }{\lambda \delta }}{3n} + \sqrt{\frac{ 2\ \Fla \log \frac{4\,\mathrm{Tr } \D }{\lambda \delta }}{n}} \right),
\end{align*}
with probability at least $1-\delta$.
\end{lemma}
And now we prove the auxiliary lemmas.
\begin{lemma}
\label{lemma:concentration_delta}
For any $0<\lambda < \|\D\| $ and any $\delta \in (0,1]$,
\begin{align*}
\left\| \D_\lambda^{-1/2}  (\widehat{\D}-\D) \D_\lambda^{-1/2} \right\| \leqslant \frac{4 \Fla \log \frac{4\,\mathrm{Tr } \D }{\lambda \delta }}{3n} + \sqrt{\frac{ 2\ \Fla \log \frac{4\,\mathrm{Tr } \D }{\lambda \delta }}{n}},
\end{align*}
with probability at least $1-\delta$.
\end{lemma}

\begin{proof}[Proof of Lemma \ref{lemma:concentration_delta}]
 
As in the proof of Lemma \ref{lemma:concentration_C}, we want to apply some concentration inequality to the operator $\D_\lambda^{-1/2} \widehat{\Delta}\D_\lambda^{-1/2}$, whose mean is exactly $\D_\lambda^{-1/2} \Delta\D_\lambda^{-1/2}$. The proof is almost the same as Lemma~\ref{lemma:concentration_C}. We start by writing
\begin{align*}
\left\|\D_\lambda^{-1/2}  (\widehat{\D}-\D) \D_\lambda^{-1/2}\right\| &= \left\|\D_\lambda^{-1/2} \widehat{\D} \D_\lambda^{-1/2} - \D_\lambda^{-1/2} \D \D_\lambda^{-1/2} \right\| \\ 
&= \left\| \frac{1}{n}\sum_{i=1}^n \left[ (\D_\lambda^{-1/2} \nabla K_{x_i})\otimes (\D_\lambda^{-1/2}\nabla K_{x_i}) - \D_\lambda^{-1/2} \Delta \D_\lambda^{-1/2} \right] \right\|. 
\end{align*}
In order to use Proposition \ref{prop:Bernstein_operator}, we bound for $i \in \llbracket 1,n\rrbracket$,
\begin{align*}
\left\|(\D_\lambda^{-1/2} \nabla K_{x_i})\otimes (\D_\lambda^{-1/2} \nabla K_{x_i}) - \D_\lambda^{-1/2} \D \D_\lambda^{-1/2}\right\| &\leqslant \left\| \D_{\lambda}^{-1/2}\nabla K_{x_i} \right\|^2_\h + \left\|\D_\lambda^{-1/2} \D \D_\lambda^{-1/2}\right\| \\
 &\leqslant 2\Fla,
\end{align*}
and, for the second order moment,
\begin{align*}
&\E \left[\left((\D_\lambda^{-1/2} \nabla K_{x_i})\otimes (\D_\lambda^{-1/2} \nabla K_{x_i}) - \D_\lambda^{-1/2} \D \D_\lambda^{-1/2}\right)^2\right] \\
&= \E \left[\left\| \D_{\lambda}^{-1/2} \nabla K_{x_i} \right\|^2_\h  (\D_\lambda^{-1/2} \nabla K_{x_i})\otimes (\D_\lambda^{-1/2} \nabla K_{x_i})\right] - \D_\lambda^{-1/2} \D \D_\lambda^{-1} \D \D_\lambda^{-1/2} \\
& \preccurlyeq \Fla \D_\lambda^{-1/2} \D \D_\lambda^{-1/2}.
\end{align*}
We conclude by some estimation of $\beta = \log \frac{2\,\mathrm{Tr } (\D \D_\lambda^{-1})}{\left\|\D_\lambda^{-1} \D \right\| \delta }$. Since $\mathrm{Tr } (\D \D_\lambda^{-1}) \leqslant \lambda^{-1} \mathrm{Tr } \D $ and for $\lambda \leqslant \|\D\|$, $\left\|\D_\lambda^{-1} \D \right\| \geqslant 1/2$, it follows that $\beta \leqslant \log \frac{4\,\mathrm{Tr } \D }{\lambda \delta }$. The conclusion then follows from \eqref{eq:concentration_operator}.
\end{proof}

\begin{lemma}[Bounding operators]
\label{lemma:magic_2}
For any $\lambda > 0$, $\delta \in (0,1)$, and $n \geqslant 15 \Fla \log \frac{4\,\mathrm{Tr } \D }{\lambda \delta }$, it holds with probability at least $1-\delta$:
\begin{align*}
\left\|\widehat{\D}_\lambda^{-1/2} \D_\lambda^{1/2} \right\|^2 \leqslant 2,
\end{align*}
\end{lemma}
The proof of this result relies on the following lemma (see proof by \cite[Proposition 8]{rudi2017generalization}).
\begin{lemma}
\label{lemma:auxi_lemma}
Let $\h$ be a separable Hilbert space, $A$ and $B$ two bounded self-adjoint positive linear operators on $\h$ and $\lambda > 0$. Then
\begin{align*}
\left\|(A+\lambda I )^{-1/2} (B+\lambda I)^{1/2}\right\|\leqslant (1-\beta)^{-1/2},
\end{align*}
with $\beta = \lambda_{\rm{max}}\left((B+\lambda I)^{-1/2} (B-A) (B+\lambda I)^{-1/2}\right) < 1$, where $\lambda_{\rm{max}}(O)$ is the largest eigenvalue of the self-adjoint operator $O$.
\end{lemma}
We can now write the proof of Lemma~\ref{lemma:magic_2}.
\begin{proof}[Proof of Lemma \ref{lemma:magic_2}]
Thanks to Lemma \ref{lemma:auxi_lemma}, we see that $$\left\|\widehat{\D}_\lambda^{-1/2} \D_\lambda^{1/2} \right\|^2 \leqslant \left(1-\lambda_{\mathrm{max}}\left(\D_\lambda^{-1/2} (\widehat{\D} - \D) \D_\lambda^{-1/2}\right)\right)^{-1},$$ and as $ \left\|\D_\lambda^{-1/2} (\widehat{\D} - \D) \D_\lambda^{-1/2}\right\| < 1 $, we have: $$\left\|\widehat{\D}_\lambda^{-1/2} \D_\lambda^{1/2} \right\|^2 \leqslant \left(1-\left\|\D_\lambda^{-1/2} (\widehat{\D} - \D) \D_\lambda^{-1/2}\right\|\right)^{-1}.$$ We can then apply the bound of Lemma \ref{lemma:concentration_delta} to obtain that, if $\lambda$ is such that $\frac{4 \Fla \log \frac{4\,\mathrm{Tr } \D }{\lambda \delta }}{3n} + \sqrt{\frac{ 2\ \Fla \log \frac{4\,\mathrm{Tr } \D }{\lambda \delta }}{n}} \leqslant \frac{1}{2}$, then $\left\|\widehat{\D}_\lambda^{-1/2} \D_\lambda^{1/2} \right\|^2 \leqslant 2$ with probability $1-\delta$. The condition on $\lambda$ is satisfied when $n \geqslant 15 \Fla \log \frac{4\,\mathrm{Tr } \D }{\lambda \delta }$.
\end{proof}

\subsubsection{Bound on the third and last term}

Here, we want to bound the term $\left \|\D_\lambda^{-1/2} \CCov (\widehat{\D}_\lambda^{-1/2} -  \D_\lambda^{-1/2}) \right\|$. Let us apply the same tricks as previously.
\begin{align*}
    \left \|\D_\lambda^{-1/2} \CCov (\widehat{\D}_\lambda^{-1/2} -  \D_\lambda^{-1/2}) \right\| &= \left \|\D_\lambda^{-1/2} \CCov \D_\lambda^{-1/2} (\D - \widehat{\D}) \widehat{\D}_\lambda^{-1/2} \right\| \\
    &\leq  \left(\Poinca \left\|\CCov\right\|\right)^{1/2}  \left\|\widehat{\D}_\lambda^{-1/2}\D_\lambda^{1/2} \right\|^2\left\|\D_\lambda^{-1/2}(\widehat{\D}_\lambda - \D_\lambda) \D_\lambda^{-1/2} \right\| . 
\end{align*}
Hence, the same bound as Lemma \ref{lemma:concentration_left} applies!

\section{The bias term}

\subsection{Proof of the consistency: Proposition~\texorpdfstring{\ref{prop:consistency_proved}}{f}}

To prove Proposition~\ref{prop:consistency_proved}, we first need a general result on operator norm convergence.
\begin{lemma}
\label{le:compact_convergence}
Let $\h$ be a Hilbert space and suppose that $(A_n)_{n \geqslant 0}$ is a family of bounded operators such that $\forall n \in \N$, $\|A_n\| \leqslant 1$ and $\forall f \in \h$, $A_n f \xrightarrow{n\to\infty} A f$. Suppose also that $B$ is a compact operator. Then, in operator norm, $$ A_n B A_n^* \xrightarrow{n\to\infty} A B A^*.$$
\end{lemma}
\begin{proof}
Let $\varepsilon > 0$. As $B$ is compact, it can be approximated by a finite rank operator $B_{_{n_\varepsilon}} = \sum_{i=1}^{n_\varepsilon} b_i \langle f_i, \cdot \rangle g_i $, where  $(f_{i})_i$  and $(g_{i})_i$  are orthonormal bases, and  $(b_{i})_i$  is a sequence of nonnegative numbers with limit zero (singular values of the operator). More precisely, $n_\varepsilon$ is chosen so that 
$$ \| B - B_{_{n_\varepsilon}} \| \leqslant \frac{\varepsilon}{2}.$$
Moreover, $\varepsilon$ being fixed, $A_n B_{_{n_\varepsilon}} A_n^* = \sum_{i=1}^{n_\varepsilon} b_i \langle A_n f_i, \cdot \rangle A_n g_i \underset{n\infty}{\longrightarrow} \sum_{i=1}^{n_\varepsilon} b_i \langle A f_i, \cdot \rangle A g_i = A B_{_{n_\varepsilon}} A^* $ in operator norm, so that, for $n \geqslant N_\varepsilon$, with $N_\varepsilon \geqslant n_\varepsilon$ sufficiently large, $\|A_n B_{_{n_\varepsilon}} A_n^* - A B_{_{n_\varepsilon}} A^*\| \leqslant \frac{\varepsilon}{2}$. Finally, as $\|A\| \leqslant 1$, it holds, for $n \geqslant N_\varepsilon$
\begin{align*}
\|A_n B_{_{n_\varepsilon}} A_n^* - A B A^*\| &\leqslant \| A_n B_{_{n_\varepsilon}} A_n^* - A B_{_{n_\varepsilon}} A^* \| + \| A ( B_{_{n_\varepsilon}} - B) A^* \| \\
& \leqslant   \| A_n B_{_{n_\varepsilon}} A_n^* - A B_{_{n_\varepsilon}} A^* \| + \| B_{_{n_\varepsilon}} - B \| \leqslant \varepsilon.
 \end{align*} 
This proves the convergence in operator norm of $A_n B A_n^*$ to $A B A^*$ when $n$ goes to infinity.
\end{proof}
We can now prove Proposition~\ref{prop:consistency_proved}.
\begin{proof}[Proof of Proposition~\ref{prop:consistency_proved}]
Denoting by $B = \D^{-1/2}\CCov\D^{-1/2}$ and by $A_\lambda = \D_\lambda^{-1/2}\D^{1/2}$ both defined on $\h_0$, we have $\D_\lambda^{-1/2}\CCov\D_\lambda^{-1/2} = A_\lambda B A_\lambda^*$ with $B$ compact and $\|A_\lambda\| \leqslant 1$. Furthermore, let $(\phi_i)_{i \in \N}$ be an orthonormal family of eigenvectors of the compact operator $\D$ associated to eigenvalues $(\nu_i)_{i \in \N}$. Then we can write, for any $ f \in \h_0$, $$A_\lambda f = \D_\lambda^{-1/2}\D^{1/2} f = \sum_{i=0}^\infty \sqrt{\frac{ \nu_i}{\lambda + \nu_i}} \langle f, \phi_i\rangle_\h \, \phi_i \underset{\lambda \rightarrow 0}{\longrightarrow} f. $$ Hence by applying Lemma~\ref{le:compact_convergence}, we have the convergence in operator norm of $\D_\lambda^{-1/2}\CCov\D_\lambda^{-1/2}$ to $\D^{-1/2}\CCov\D^{-1/2}$.
\end{proof}

\subsection{Fast rates under source condition: \texorpdfstring{Proposition~\ref{prop:bias_fast_rates}}{f}}

\begin{proof}[Proof of Proposition~\ref{prop:bias_fast_rates}] To show Proposition~\ref{prop:bias_fast_rates}, we simply bound the bias term according to the following inequalities.
\begin{align*}
    \left\|\Pi^p \left(\D_\lambda^{-1/2} \CCov \D_\lambda^{-1/2} - \D^{-1/2} \CCov \D^{-1/2}\right) \Pi^p\right\| &\leq \left\|\Pi^p \left(\D_\lambda^{-1/2} - \D^{-1/2}\right) \CCov \D_\lambda^{-1/2} \Pi^p\right\| \\
    &\hspace{3.79cm}+ \left\|\Pi^p  \D^{-1/2} \CCov \left(\D_\lambda^{-1/2} - \D^{-1/2}\right) \Pi^p\right\| \\
     &\leq \left\|\Pi^p \D_\lambda^{-1/2} \D^{-1/2} \CCov \D^{-1/2} \Pi^p\right\| + \left\|\Pi^p  \D^{-1/2} \CCov \D^{-1/2}\D_\lambda^{-1/2} \Pi^p\right\| \\
      &\leq \left\|\Pi^p \D^{-1/2} \D^{-1/2} \CCov \D^{-1/2} \Pi^p\right\| + \left\|\Pi^p  \D^{-1/2} \CCov \D^{-1/2}\D^{-1/2} \Pi^p\right\| \\
      &\leq 2 \Poinca  \left\|\Pi^p \D^{-1/2} \Pi^p\right\|,
\end{align*}
which finally proves Proposition~\ref{prop:bias_fast_rates}.
\end{proof}
Of course, under refined a priori on how smooth are the eigenvectors of $\mathcal{L}$, i.e., on control like $\left\|\Pi^p \D^{-\theta} \Pi^p\right\|$, for $\theta \in [0,1]$ that generalize the \textit{source conditions} we used, we could get finer-grained rates~\cite{pillaud2018statistical, berthier2020tight,varre2021last}.

\end{document}